\documentclass[preprint,12pt,number]{elsarticle}

\usepackage[T1]{fontenc}
\usepackage[utf8]{inputenc}
\usepackage{lmodern}
\usepackage{amsmath,amssymb,mathtools}
\usepackage{amsthm}
\usepackage{booktabs,longtable,array}
\usepackage{enumitem}
\usepackage{graphicx}
\usepackage{xcolor}
\usepackage{float}
\usepackage[hidelinks]{hyperref}
\usepackage{tikz}
\usetikzlibrary{arrows.meta,positioning,shapes.misc,fit,backgrounds,calc}
\usepackage{caption}
\newcolumntype{L}[1]{>{\raggedright\arraybackslash}p{#1}}

\definecolor{cprec}{HTML}{1f6f8b}
\definecolor{cset}{HTML}{8e44ad}
\definecolor{cenv}{HTML}{c0392b}
\definecolor{crule}{HTML}{27ae60}
\definecolor{cthr}{HTML}{b9770e}

\setlist{nosep}
\providecommand{\doi}[1]{\textsc{doi:}\,#1}

\newcommand{\A}{\mathcal A}
\newcommand{\Pset}{\mathcal P}
\newcommand{\Hset}{\mathcal H}
\newcommand{\Dset}{\mathcal D}
\newcommand{\E}{\mathbb E}
\newcommand{\pboxLowerE}{\underline{\mathbb E}^{\square}}
\newcommand{\pboxUpperE}{\overline{\mathbb E}^{\square}}
\newcommand{\NB}{\operatorname{NB}}
\newcommand{\EVPI}{\operatorname{EVPI}}
\newcommand{\EVPPI}{\operatorname{EVPPI}}
\newcommand{\EVSI}{\operatorname{EVSI}}
\newcommand{\ENBS}{\operatorname{ENBS}}
\newcommand{\Cstudy}{C_{\mathrm{study}}}
\newcommand{\VOI}{\operatorname{VOI}}

\newcommand{\argmax}{\operatorname*{arg\,max}}
\newtheorem{theorem}{Theorem}
\newtheorem{proposition}{Proposition}
\theoremstyle{remark}
\newtheorem{remark}{Remark}
\theoremstyle{plain}

\newcounter{algctr}
\newenvironment{workflowalgorithm}[1]{%
 \refstepcounter{algctr}%
 \par\medskip\noindent\textbf{Algorithm \thealgctr. #1}\par\smallskip%
 \begin{enumerate}[leftmargin=2.1em,label=\arabic*.]
}{%
 \end{enumerate}\medskip
}

\journal{International Journal of Approximate Reasoning}

\begin{document}

\begin{frontmatter}

\title{Value of Information under Imprecise Probabilities:\\
Decision Rule-Specific Values and Fixed-Measure Envelopes on a Credal Set}

\author[mtl,brown]{Rowan Iskandar}
\ead{rowan.iskandar@gmail.com}

\affiliation[mtl]{organization={Medtronic Trading S\`arl},
  city={Tolochenaz}, state={VD}, country={Switzerland}}
\affiliation[brown]{organization={Center for Evidence Synthesis in Health, Brown University School of Public Health},
  city={Providence}, state={RI}, country={USA}}

\begin{abstract}
Value-of-information (VOI) analysis is usually conducted under a single
probability measure.
However, in practice, the available evidence often pins the measure down
only to a set. 
Consequently, under a set of probability measures, VOI requires different formulations.
First, we explicate a \emph{rule-specific VOI} that fixes a decision rule for acting under
imprecision (such as $\Gamma$-maximin) and measures what the information is worth
to a decision maker who uses that rule.
Second, we derive a \emph{fixed-measure envelope} that evaluates the classical VOI functional over all admissible precise measures. 
We formalize this distinction and explicate its consequences for the expected perfect, partial, and sample information. 
The expected value of perfect information is concave over the credal set. 
Hence, when the set is generated by finitely many measures, its lower envelope endpoint is
obtained exactly from the generators, while its upper endpoint may be interior
and is computed by a finite linear program. The $\Gamma$-maximin value, in
contrast, can exceed the entire envelope, so a rule-specific value is not
recovered from the envelope's endpoints. 
A continuity bound limits how much the VOI can change as the measure varies, and we identify when the partial- and sample-information endpoints can still be obtained from the generators.
Because the single-measure VOI must itself be estimated, the procedure we give combines standard estimators for it with a search over the credal set.
By using a worked decision problem, we show how the
two quantities separate conclusions that hold across every admissible measure
from conclusions that depend on one unidentified choice of measure.
\end{abstract}


\end{frontmatter}

\section{Introduction}
\label{sec:intro}

A decision-maker often must decide before the evidence is complete.
Such a decision-making situation, often formulated as a decision model, is fraught with uncertainty. 
The uncertainty may induce loss to the decision-maker due to choosing a suboptimal decision. 
Value-of-information (VOI) analysis then
asks how much the decision could improve if some or all of that uncertainty were
resolved before the choice is made
\citep{felli1998sensitivity,ades2004evsi,jackson2019voi}. 
This question has a
clear answer when a single probability measure is accepted as the basis for characterizing uncertainty.
However, there are situations when the evidence does not single out
one such measure. For example, the results of two evidence syntheses may both be defensible yet disagree.
Committing to one probability measure then
makes the VOI analysis results look more precise than the evidence warrants.
One way to relax this restriction is to consider 
every plausible measure in the analysis, instead of choosing
one. 
A credal set $\mathcal P$ is a set of probability measures, each
compatible with the evidence \citep{walley1991,augustin2014introduction}.
Such
sets naturally arise from interval or bound constraints, partially specified moments,
weakly identified dependence, robust classes of priors, or a finite set of
acceptable model structures
\citep{jackson2010structural,jackson2011framework,iskandar2026representations}.
This has a direct consequence for VOI. 
A single measure plays two roles: it ranks actions by their expected net benefit, and it determines the distribution over which the VOI is computed. 
In contrast, a credal set determines neither a single ranking of the actions nor a single VOI value.
We must then distinguish how much the information is worth for the decision once a rule for acting under several measures is fixed, from a sensitivity summary, which captures how the VOI calculation moves as
the measure varies.

The first quantity is the value of the information under a stated rule for acting under imprecision, such as \(\Gamma\)-maximin, which ranks each action by its worst expected net benefit over \(\mathcal P\), minimax regret, or a weighted lower-upper rule.
Before the information $Z$ is observed, the decision maker can
commit only to a single fixed action.
After $Z$ is observed, the decision-maker
can choose an act that depends on the information. 
This rule-specific VOI quantifies the gain in value, measured by the chosen rule, from being able to adapt one's act to $Z$. 
This quantity answers a decision question: what the information is worth to someone who acts by that rule. 
Set-valued criteria, such as E-admissibility or maximality, need further selection or reporting
convention before they yield a single VOI number.
The second quantity is the range of the standard fixed-measure VOI over the credal set.
The analyst fixes one VOI target and evaluates it under each measure in the credal set. 
Writing $T(P)$ for that classical VOI value under measure $P$, the fixed-measure VOI envelope is
\begin{equation*}
  [\inf_{P\in\mathcal P} T(P),\ \sup_{P\in\mathcal P} T(P)].
\end{equation*}
Here ``envelope'' refers to the lowest and highest classical VOI
values over the measures compatible with the evidence. 
Comparing this interval with a decision threshold asks whether a research conclusion would change if a different admissible measure had been used as the prior.
These two quantities answer different questions.
The rule-specific value is not necessarily equal to an endpoint of the fixed-measure envelope. In fact, it can even lie outside the whole envelope. 
Reporting both prevents a VOI computed under a single, unsupported measure from being mistaken for a robustness analysis.

This distinction builds on established decision-theoretic ideas. 
In  statistical decision theory, Bayes values and risks have an affine/convex structure
as functions of the prior \citep{degroot1970optimal,berger1985statistical}.
In applied work in Health, VOI is a standard tool for sensitivity analysis and research prioritization in medical decision making \citep{felli1998sensitivity,jackson2019voi,claxton2006prioritise,fenwick2020ispor1,rothery2020ispor2,jackson2022healthpolicy}. 
The lower-expectation rule used below is the \(\Gamma\)-maximin form of maxmin expected utility with a non-unique prior \citep{gilboa1989maxmin}. 
Related work studies information beyond the single-prior expected-utility setting.
\citet{szaniawski1967perfect} studied information values under maximin,
minimax, and Hurwicz criteria; \citet{wakker1988nonexpected} showed that
nonexpected-utility preferences can make information aversive.
\citet{schlee1991perfect} gave a weak dominance condition under which fully
resolving uncertainty remains nonnegative; \citet{snow2010ambiguity} studied ambiguity
and the value of information; and dilation \citep{seidenfeld1993dilation}
shows that conditioning can widen a credal set. This paper uses those ideas to
study VOI over credal sets. The rule-specific decision value and the
fixed-measure envelope are different estimands, have different endpoint
structures, and support different threshold interpretations.

This paper is organized as follows.
Section~\ref{sec:framework}
defines the two quantities precisely. Section~\ref{sec:theory} works out their
mathematical structure, including how they are computed when the credal set is
described by finitely many measures: one envelope endpoint follows directly,
while the other requires solving an optimization problem.
Section~\ref{sec:implementation} gives a procedure for computing both quantities in practice, separating the parts that can be computed exactly from those that
must be estimated numerically. Section~\ref{sec:application} works through a
decision problem in detail, showing how the two quantities distinguish
conclusions that hold across every admissible measure from conclusions that
depend on one particular choice of measure. Figure~\ref{fig:concept} summarizes
the two quantities and how each is compared with a decision threshold.
\begin{figure}[H]
\centering
\hyphenpenalty=10000\exhyphenpenalty=10000\relax
\resizebox{\linewidth}{!}{%
\begin{tikzpicture}[
  font=\small,
  >={Stealth[length=2.2mm]},
  box/.style={rounded corners=2pt,draw=#1,line width=0.7pt,fill=#1!6,
              text width=31.5mm,align=center,inner sep=4pt,minimum height=14mm},
  rulebox/.style={box=crule,text width=35mm},
  lab/.style={font=\scriptsize\itshape,text=black!60},
  ar/.style={->,line width=0.8pt,draw=black!55},
]
\node[box=cprec] (prec) {\textbf{Classical VOI}\\[2pt]\scriptsize one precise measure $P_0$; one selected information target};
\node[box=cset,right=14mm of prec] (set) {\textbf{Credal set} $\mathcal P$\\[2pt]\scriptsize every measure compatible with the evidence};
\node[box=cenv,above right=6mm and 16mm of set] (env) {\textbf{Fixed-measure envelope}\\[2pt]\scriptsize $[\underline T,\overline T]$: range of classical VOI over $\mathcal P$};
\node[rulebox,below right=6mm and 16mm of set] (rule) {\textbf{\mbox{Rule-specific VOI}}\\[2pt]\scriptsize value under a criterion (e.g.\ lower expectation)};
\node[box=cthr,right=60mm of set] (thr) {\textbf{Threshold classification}\\[2pt]\scriptsize compare with decision threshold $\tau$};
\draw[ar] (prec) -- node[lab,above]{relax} (set);
\draw[ar] (set.east) to[out=32,in=180] (env.west);
\draw[ar] (set.east) to[out=-32,in=180] (rule.west);
\draw[ar] (env.east) to[out=0,in=120] (thr.north);
\draw[ar] (rule.east) to[out=0,in=-120] (thr.south);
\node[draw=black!25,rounded corners=2pt,fill=black!2,inner sep=6pt,
      text width=132mm,font=\scriptsize,align=left,below=12mm of rule.south,xshift=-30mm] (out) {
  \textbf{Reading the result against the threshold $\tau$}\\[2pt]
  \textcolor{crule!70!black}{$\blacksquare$~\textbf{Robust}} --- the whole envelope lies on one side of $\tau$
  ($\underline T>\tau$, informative across $\mathcal P$; or $\overline T<\tau$, not informative across $\mathcal P$).\\[1pt]
  \textcolor{cthr}{$\blacksquare$~\textbf{Measure-dependent}} --- the envelope straddles $\tau$
  ($\underline T\le\tau\le\overline T$), some admissible measures cross the threshold; others do not.
};
\end{tikzpicture}%
}
\caption{Conceptual overview. A single reference measure $P_0$ is relaxed to a
credal set $\mathcal P$. The fixed-measure envelope $[\underline T,\overline T]$
records the range of a classical VOI functional over $\mathcal P$, whereas a
rule-specific VOI evaluates information under a chosen criterion for imprecision.
Envelope-threshold classification addresses robustness across precise measures;
a rule-specific threshold comparison answers a separate criterion-dependent
question. Boundary cases ($\underline T=\tau$ or $\overline T=\tau$) are
threshold-sensitive.}
\label{fig:concept}
\end{figure}
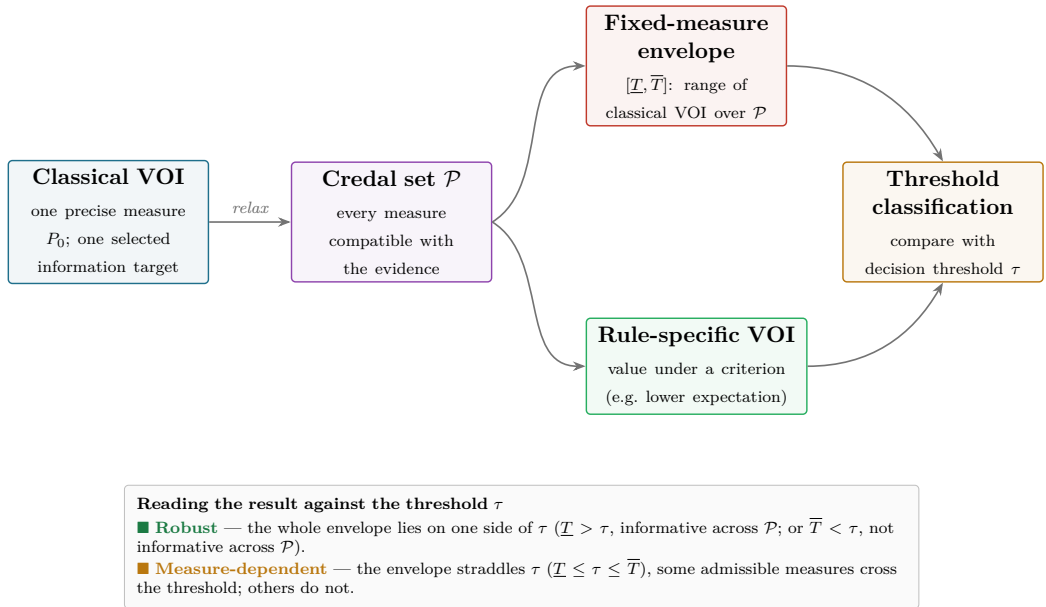

\section{Value of information on a credal set}
\label{sec:framework}

This section sets up the value of information on a credal set in three parts. We first recall the classical VOI functionals under a single precise measure, fixing notation and isolating where the choice of measure enters. We then replace that measure by a credal set and define the rule-specific and envelope estimands, stating the structural properties that make them usable. Finally, we give the unified estimation procedure. Each result is proved where stated, with supporting results given in \ref{app:proofs}. 

\subsection{Classical VOI under a single measure}
\label{sec:standard-voi}
\label{sec:psa-start}

Consider a decision model with a finite set of strategies $\A$, with $|\A|\ge 2$ and a generic strategy $a\in\A$. Let $\theta\in\Theta$ denote the uncertain model inputs, where $\Theta$ is the parameter space. For each strategy, the model returns a net benefit $\NB_a(\theta)\in\mathbb R$, with larger values preferred. This notation covers any value function or negative loss that can be evaluated for each strategy. In decision models that include both health outcomes and resource use, one common decision-relevant metric is $\NB_a(\theta;\lambda)=\lambda Q_a(\theta)-C_a(\theta)$, where $Q_a(\theta)\in\mathbb R$ is the effect measure, $C_a(\theta)\in\mathbb R$ is resource use or cost, and $\lambda\in(0,\infty)$ is the value assigned to one unit of effect. Throughout, $\NB_a(\theta)$ denotes this net benefit.

In a probabilistic sensitivity analysis, the standard approach for uncertainty propagation, the analyst specifies one precise measure $P$ on $\Theta$, samples input vectors from $P$, and evaluates each strategy's net benefit. The VOI quantities below are written directly as expectations under $P$.

The current strategy is the one with the highest expected net benefit under the specified measure. Ties can be handled by any prespecified rule.
\begin{equation}
 a^\ast=\argmax_{a\in\A}\E_P\{\NB_a(\theta)\}.
\end{equation}
Here $\E_P\{\cdot\}$ denotes expectation under $P$, and all VOI expectations are assumed finite.

EVPI compares the current decision, taken before the uncertainty in $\theta$ is resolved, with a perfect-information world in which $\theta$ is learned first and the best strategy is selected at each realized value.
\begin{equation}
 \EVPI(P)=
 \E_P\left[\max_{a\in\A}\NB_a(\theta)\right]
 -
 \max_{a\in\A}\E_P\left[\NB_a(\theta)\right].
 \label{eq:psa-evpi}
\end{equation}
This is a loss-from-uncertainty calculation, the expected net benefit lost because the decision must be made before the state of the world is known.

EVPPI and EVSI use the same logic with less-than-perfect information. Let $\phi=\phi(\theta)\in\Phi$ be a parameter group or function of the uncertain inputs, where $\Phi$ is its range. If $\phi$ were learned perfectly before decision making, the analyst would choose the strategy with the highest expected net benefit conditional on $\phi$.
\begin{equation}
 \EVPPI(P;\phi)=
 \E_P\left[\max_{a\in\A}\E_P\{\NB_a(\theta)\mid \phi\}\right]
 -
 \max_{a\in\A}\E_P\left[\NB_a(\theta)\right].
 \label{eq:psa-evppi}
\end{equation}
If a proposed study design $d\in\Dset$ would generate future data $Y_d\in\mathcal Y_d$, where $\Dset$ is the candidate design set and $\mathcal Y_d$ is the sample space for the future data, then
\begin{equation}
 \EVSI(P;d)=
 \E_P\left[\max_{a\in\A}\E_P\{\NB_a(\theta)\mid Y_d\}\right]
 -
 \max_{a\in\A}\E_P\left[\NB_a(\theta)\right].
 \label{eq:psa-evsi}
\end{equation}
The expected net benefit of sampling is $\ENBS(P;d)=\EVSI(P;d)-\Cstudy(d)$, where $\Cstudy(d)\in[0,\infty)$ is the fixed cost of study design $d$ unless otherwise stated.

Equations~\eqref{eq:psa-evpi}--\eqref{eq:psa-evsi} are the reference definitions. The extension below keeps them but considers how the same quantities should be read when the evidence supports an admissible set rather than a single $P$.

\subsection{From a single measure to a credal set}
\label{sec:ip-voi}

The previous section assumed that the probability measure $P$ can be precisely specified. Value of information under imprecise probability (IP-VOI) becomes relevant when that assumption cannot be defended. The goal is not to replace the logic and machinery of VOI, but to honestly represent uncertainty using a set of probability measures supported by the evidence.
The VOI task is then to quantify the value of research across that set.

\subsubsection{Credal sets and p-boxes}
\label{sec:why-ip}

In IP-VOI the first step is to state what current evidence allows. Instead of
committing to a single probability measure, as a standard probabilistic sensitivity
analysis does, the analyst defines
a nonempty set $\Pset$ of admissible probability measures on $\Theta$. Each
$P\in\Pset$ is a complete measure that can be used in the standard decision model,
but the evidence does not justify selecting one of them as uniquely correct.
Examples include alternative evidence syntheses, structural or extrapolation
models, weakly identified dependence, interval-valued moments or quantiles, and
partially specified distributions.

Probability boxes, or p-boxes, are one useful way to encode such imprecision for
scalar inputs whose CDFs are only partially known \citep{ferson2002constructing,ferson2006sensanal,iskandar2021pba}. For a scalar input $X$ with possible CDF
$F_X$, a p-box is a pair of valid CDF bounds satisfying
\begin{equation}
 \underline F_X(x) \le F_X(x) \le \overline F_X(x),
 \qquad x\in\mathbb R .
 \label{eq:pbox-definition}
\end{equation}
Probability bounds analysis propagates such bounds, together with explicit
dependence assumptions, through the decision model. It is not a separate VOI
criterion here, but one route to specifying or approximating the credal set.

\subsubsection{Distributional bounds and expectation bounds}
\label{sec:pbox-to-expectation}

For a strategy $a$, let $X_a=\NB_a(\theta)$ and
$F_{a,P}(x)=P\{\NB_a(\theta)\le x\}$. Varying $P$ over $\Pset$ induces a set of
output distributions with pointwise bounds
\[
\underline F_a(x)=\inf_{P\in\Pset}F_{a,P}(x),\qquad
\overline F_a(x)=\sup_{P\in\Pset}F_{a,P}(x).
\]
Decision making, however, uses expectations. For any bounded output $X$,
\begin{equation}
 \underline{\E}_{\Pset}(X)=\inf_{P\in\Pset}\E_P(X),
 \qquad
 \overline{\E}_{\Pset}(X)=\sup_{P\in\Pset}\E_P(X).
 \label{eq:lower-upper-expectation}
\end{equation}
A full p-box band on support $[\ell_X,u_X]$ is
\[
\begin{aligned}
\mathcal B(\underline F_X,\overline F_X)=
\{F:\;&F \text{ is a CDF supported on }[\ell_X,u_X],\\
&\underline F_X(x)\le F(x)\le \overline F_X(x)\ \forall x\}.
\end{aligned}
\]
When the bounds are valid CDFs, the lower and upper expectations over this full
band are (Proposition~\ref{prop:pbox-expectation})
\begin{equation}
\begin{split}
 \pboxLowerE(X)&=\ell_X+\int_{\ell_X}^{u_X}\{1-\overline F_X(x)\}\,dx,\\
 \pboxUpperE(X)&=\ell_X+\int_{\ell_X}^{u_X}\{1-\underline F_X(x)\}\,dx .
\end{split}
\label{eq:pbox-expectation}
\end{equation}
Because the p-box band can contain distributions not induced by any
$P\in\Pset$,
\begin{equation}
 \pboxLowerE(X)\le \underline{\E}_{\Pset}(X)
 \le \overline{\E}_{\Pset}(X)\le \pboxUpperE(X),
 \label{eq:pbox-expectation-outer}
\end{equation}
with equality only when the admissible output distributions are the full band.
For EVPI, p-box propagation must preserve joint structure. The quantity
$M(\theta)=\max_a\NB_a(\theta)$ depends on the joint vector of strategy net
benefits, not on their marginal output p-boxes separately.

\subsubsection{Decision criteria and rule-specific VOI}
\label{sec:decision-rule-ip}

Once expected net benefits are interval-valued, identifying ``the strategy with the highest expected net benefit'' is no longer straightforward, because the credal set determines only an interval of expected net benefit for each strategy, not a single value, so there is no unique ranking of the strategies. A decision rule specifying how a strategy or policy is selected from among interval-valued expected net benefits is therefore part of the VOI estimand.

Let $Z$ denote information observed before the final strategy is chosen. For a current decision, $Z$ is degenerate and a policy is simply a constant choice $a\in\A$. For EVPPI or EVSI, a policy is a measurable function $\delta:\mathcal Z\to\A$, where $\mathcal Z$ is the range of $Z$. Let $\Delta_Z$ denote the class of policies under consideration. In the unrestricted finite case, $\Delta_Z=\{\delta:\mathcal Z\to\A\}$ denotes all measurable maps. Policies and actions are \emph{deterministic} throughout. The set $\A$ contains the named, implementable strategies, and randomized mixtures are admitted only if they are added to $\A$ as explicit strategies. This is the natural convention for health-economic decision models, whose strategies are discrete and often not randomizable, but it is a modeling choice. Enlarging $\A$ with randomized strategies can change the rule-specific lower-expectation values (it cannot change any single-measure or fixed-measure-envelope quantity, since those maximize a linear objective that is already optimized at a pure action), as we note for the overshoot example in Section~\ref{sec:concavity-thm}.
Let $\Pset_Z$ be the admissible set of joint probability measures for $(\theta,Z)$, extending $\Pset$ in the sense that the set of $\theta$-marginals of measures in $\Pset_Z$ is exactly $\Pset$. For the two information types this joint set is not arbitrary. For EVPPI, $Z=\phi(\theta)$ is a function of $\theta$, so each $P\in\Pset$ induces a single joint law by push-forward and $\Pset_Z=\{P\circ(\mathrm{id},\phi)^{-1}:P\in\Pset\}$. For EVSI with design $d$, the future data are generated through the sampling kernel $K_d(\mathrm dy\mid\theta)$ fixed by the design, so $\Pset_{Y_d}=\{P(\mathrm d\theta)\,K_d(\mathrm dy\mid\theta):P\in\Pset\}$. Only joint laws of this form are admissible, which ties EVSI to the specific data-generating design rather than to an arbitrary joint extension with the correct $\theta$-marginal. A decision rule $\rho$ assigns each policy a scalar criterion $C_\rho(\delta;\Pset_Z)$. The value after observing $Z$ under rule $\rho$ is
\begin{equation}
 V_Z^{\rho}(\Pset_Z)=
 \sup_{\delta\in\Delta_Z} C_\rho(\delta;\Pset_Z),
 \label{eq:general-rule-value}
\end{equation}
with the current value $V_0^{\rho}(\Pset)$ obtained by restricting $\Delta_Z$ to constant strategies. The rule-specific VOI is
\begin{equation}
 \VOI_Z^{\rho}=V_Z^{\rho}(\Pset_Z)-V_0^{\rho}(\Pset),
 \label{eq:general-ipvoi}
\end{equation}
provided both values are finite. This expression covers EVPI by setting $Z=\theta$, EVPPI by setting $Z=\phi(\theta)$, and EVSI by setting $Z=Y_d$. For EVSI, the corresponding net value subtracts $\Cstudy(d)$.

Different rules answer different questions and should be fixed as part of the estimand, before results are seen. Lower expected net benefit, also called the lower-expectation or $\Gamma$-maximin criterion, is the maxmin expected-utility rule of \citet{gilboa1989maxmin} specialized to a credal set of named strategies. It uses $C_{\mathrm{LE}}(\delta;\Pset_Z)=\inf_{P\in\Pset_Z}\E_P\{\NB_{\delta(Z)}(\theta)\}$ and selects the policy with the largest guaranteed expected value. This criterion suits a decision maker who wants protection against unsupported optimism (for example, a payer or regulator requiring a conclusion that holds over every measure still compatible with the evidence), but the criterion can be conservative when $\Pset$ is broad or contains extreme measures, and should be read as a guaranteed-value analysis rather than the unique prescription of imprecise probability.

Other rules trade off differently. Minimax regret targets the largest opportunity loss relative to the best strategy under each measure rather than the worst absolute value. Maximality and interval dominance screen out dominated strategies without forcing a unique choice. E-admissibility describes the strategies optimal under at least one measure. Weighted lower--upper criteria combine the two bounds through an explicit ambiguity-attitude weight that the decision context must justify. Scalar rules such as $\Gamma$-maximin, minimax regret, and the weighted lower--upper criteria assign each policy a number $C_\rho(\delta;\Pset_Z)$ and so define a scalar rule-specific VOI directly through Equations~\eqref{eq:general-rule-value}--\eqref{eq:general-ipvoi}. Set-valued rules such as E-admissibility and maximality instead return a set of admissible policies and require an additional selection or reporting convention before they yield a scalar value, which we do not develop here. The fixed-measure envelope is not itself a rule for choosing under imprecision but a diagnostic for whether conventional VOI conclusions depend on the selected measure.

IP-VOI therefore does not require the conservatism of $\Gamma$-maximin, only transparency about the rule used to turn the imprecision into a single value. The algorithms below use the lower-expectation rule as the worked implementation because this rule is transparent, interpretable, conservative, and estimable directly from simulation output. Other rules are substituted by replacing $C_\rho$ in Equation~\eqref{eq:general-rule-value}. Under this rule,
\begin{equation}
 V_0^{\mathrm{LE}}(\Pset)=
 \max_{a\in\A}\inf_{P\in\Pset}\E_P\{\NB_a(\theta)\}.
 \label{eq:lower-current}
\end{equation}
If perfect information about $\theta$ is available and policies are unrestricted, the policy can choose the pointwise best strategy at each realized value of $\theta$, giving
\begin{equation}
 V_{\Theta}^{\mathrm{LE}}(\Pset)=
 \inf_{P\in\Pset}\E_P\left\{\max_{a\in\A}\NB_a(\theta)\right\}.
 \label{eq:lower-perfect}
\end{equation}
This equality holds because the pointwise optimal policy is the same function of $\theta$ for every $P\in\Pset$. At each realized $\theta$, this policy chooses a strategy with the largest net benefit and therefore attains the maximum inside the expectation for all admissible measures simultaneously. The same simplification does not generally extend to EVPPI or EVSI, because the best policy after observing partial information or future data may depend on which admissible probability measure is used.
Therefore, lower-expectation EVPI is
\begin{equation}
 \EVPI_{\Pset}^{\mathrm{LE}}=
 V_{\Theta}^{\mathrm{LE}}(\Pset)-V_0^{\mathrm{LE}}(\Pset).
 \label{eq:lower-evpi}
\end{equation}
For EVPPI or EVSI, Equations~\eqref{eq:general-rule-value}--\eqref{eq:general-ipvoi} are used with $Z=\phi(\theta)$ or $Z=Y_d$ and the same lower-expectation criterion. Lower-expectation EVPI is always nonnegative (Proposition~\ref{prop:robust-evpi}).

\subsubsection{The fixed-measure envelope}
\label{sec:misleading-standard-voi}

Rule-specific IP-VOI answers how much information is worth under a stated rule for decisions with imprecision. A complementary question asks how the usual VOI calculation would change if the analyst selected a different admissible precise probability measure. Let $T(P)$ denote a conventional VOI target under a precise probability measure $P$, such as $\EVPI(P)$, $\EVPPI(P;\phi)$, $\EVSI(P;d)$, or $\ENBS(P;d)$. The fixed-measure VOI envelope is
\begin{equation}
 T_{\mathrm{env}}(\Pset)=\{T(P):P\in\Pset\},
 \qquad
 \underline T=\inf_{P\in\Pset}T(P),\quad
 \overline T=\sup_{P\in\Pset}T(P).
 \label{eq:fixed-measure-envelope}
\end{equation}
The envelope is not a decision rule under imprecision but a diagnostic for how strongly the conventional VOI result depends on the precise probability measure selected for the uncertainty analysis. The set of achievable values is the sharp range left by the evidence, and forms a closed interval whenever $\Pset$ is compact and connected and $T$ is continuous (Proposition~\ref{prop:sharp}). If a reference measure $P_0\in\Pset$ is used, then $T(P_0)$ should lie inside $[\underline T,\overline T]$ up to numerical error.

The envelope identifies settings in which a single-measure VOI result can be misleading for research prioritization. Let $\tau\in\mathbb R$ denote the decision threshold for a VOI quantity, such as zero for net benefit of sampling or the minimum value required to justify a research program. If $\underline T>\tau$, the research conclusion is positive across the admissible set. If $\overline T<\tau$, the conclusion is negative across the admissible set. If $\underline T\le \tau\le \overline T$, the conclusion is assumption-dependent, since some admissible measures support the research and others do not. A single reference value is potentially misleading when the value lies on one side of $\tau$ while the envelope crosses $\tau$, or near one endpoint of a wide envelope. Two useful diagnostics are
\begin{equation}
 \Delta_{\mathrm{miss}}=\max\{0,\overline T-T(P_0)\},
 \qquad
 \Delta_{\mathrm{nonguaranteed}}=\max\{0,T(P_0)-\underline T\}.
 \label{eq:miss-guarantee}
\end{equation}
The first quantity measures how much value could be missed by reporting only the reference measure. The second measures how much of the reference value is not supported across all admissible measures.

\section{Structural theory of the value-of-information functionals}
\label{sec:theory}

This section establishes the structural properties of the three
value-of-information functionals, namely the expected value of perfect information
($\EVPI$), of partial perfect information ($\EVPPI$), and of sample information
($\EVSI$), viewed as functions of the probability measure ranging over a credal set.
The central results concern $\EVPI$ (Sections~\ref{sec:concavity-thm}
and~\ref{sec:sharp-tv-sub}). The partial- and sample-information functionals are
treated in Section~\ref{sec:vertex-fail}, and the sign of the rule-specific value
of all three in Section~\ref{sec:sign-voi}. Throughout, $\A$ is
finite, the net benefits are bounded ($|\NB_a(\theta)|\le B$ for all $a\in\A$
and $\theta\in\Theta$), and $\Pset$ is a nonempty compact convex subset of a
locally convex space of signed measures on $\Theta$, in a topology under which the finitely many
expectation maps $P\mapsto\E_P[\NB_a]$ ($a\in\A$) and $P\mapsto\E_P[M]$ are
continuous. When $\Pset=\operatorname{conv}\{G_1,
\ldots,G_K\}$ is finitely generated by $K$ measures these conditions hold automatically.
Recall $M(\theta)=\max_{a\in\A}\NB_a(\theta)$ from
Section~\ref{sec:pbox-to-expectation}, and abbreviate the two terms of the
classical EVPI as
\[
 f(P)=\E_P[M],\qquad
 g(P)=\max_{a\in\A}\E_P[\NB_a],
\]
so that the classical perfect-information and current-value functionals are
$V_\Theta(P)=f(P)$ and $V_0(P)=g(P)$, and $\EVPI(P)=f(P)-g(P)$ by
Equation~\eqref{eq:psa-evpi}. The fixed-measure envelope of a
functional $T$ over $\Pset$, defined in Equation~\eqref{eq:fixed-measure-envelope},
is written $[\,\underline T,\overline T\,]=[\inf_{P\in\Pset}T(P),\
\sup_{P\in\Pset}T(P)]$. In the notation of the present section, the
lower-expectation ($\Gamma$-maximin) quantities of
Equations~\eqref{eq:lower-current}--\eqref{eq:lower-evpi} read
$V_\Theta^{\mathrm{LE}}=\inf_{P\in\Pset}f(P)$,
$V_0^{\mathrm{LE}}=\max_{a\in\A}\inf_{P\in\Pset}\E_P[\NB_a]$, and
$\EVPI_{\Pset}^{\mathrm{LE}}=V_\Theta^{\mathrm{LE}}-V_0^{\mathrm{LE}}$, with the
argument $\Pset$ suppressed where clear from context.

\subsection{Concavity, exact lower endpoint, and the envelope--rule gap}
\label{sec:concavity-thm}

The following theorem records the consequences, for VOI envelopes over credal
sets, of the classical convexity structure of the optimal expected net benefit. The
concavity statement is the payoff-form analogue of the standard Bayes-risk
convexity/concavity result. The point here is to spell out what it implies for
fixed-measure VOI envelopes and for lower-expectation IP-VOI. In particular, the
\emph{lower} EVPI envelope endpoint is computed exactly by evaluating the
generators of a finitely generated credal set, whereas the $\Gamma$-maximin EVPI,
while never falling below the lower endpoint of the envelope, can lie strictly
\emph{above} the entire envelope. Thus the rule-specific value is not a selection
from, or a summary statistic of, the fixed-measure envelope.

\begin{theorem}[Structural consequences for EVPI over credal sets]
\label{thm:main}
Under the standing assumptions:
\begin{enumerate}
\item[(a)] \textnormal{(Concavity.)} $P\mapsto\EVPI(P)$ is concave on $\Pset$.
\item[(b)] \textnormal{(Exact lower endpoint.)} The lower envelope endpoint is
attained at an extreme point of $\Pset$. Consequently, if
$\Pset=\operatorname{conv}\{G_1,\ldots,G_K\}$, then
\[
 \underline{\EVPI}=\inf_{P\in\Pset}\EVPI(P)=\min_{1\le k\le K}\EVPI(G_k),
\]
so the lower endpoint is computed exactly by evaluating the $K$ generators,
without solving an optimization over $\Pset$.
\item[(c)] \textnormal{(Envelope--rule gap.)} The $\Gamma$-maximin EVPI always
satisfies $\EVPI_{\Pset}^{\mathrm{LE}}\ge\underline{\EVPI}$, but the inequality
$\EVPI_{\Pset}^{\mathrm{LE}}\le\overline{\EVPI}$ can fail, and there exist models in
which $\EVPI_{\Pset}^{\mathrm{LE}}>\overline{\EVPI}$. Hence the rule-specific
value need not lie within the fixed-measure envelope.
\end{enumerate}
\end{theorem}

\begin{proof}
(a) Write $D_a(\theta)=M(\theta)-\NB_a(\theta)=\max_{b\in\A}\NB_b(\theta)-
\NB_a(\theta)\ge0$. Since $\E_P[M]-\E_P[\NB_a]=\E_P[D_a]$ and the current value
is $\max_a\E_P[\NB_a]$,
\begin{equation}
 \EVPI(P)=\E_P[M]-\max_{a\in\A}\E_P[\NB_a]
 =\min_{a\in\A}\bigl(\E_P[M]-\E_P[\NB_a]\bigr)
 =\min_{a\in\A}\E_P[D_a].
 \label{eq:min-affine}
\end{equation}
Each $P\mapsto\E_P[D_a]$ is affine, so $\EVPI$ is a pointwise minimum of affine
functions and is therefore concave. (Representation \eqref{eq:min-affine} also
shows $\EVPI(P)\ge0$, since $D_a\ge0$ implies $\E_P[D_a]\ge0$.)

(b) Let $P=\sum_k w_k G_k$ be any point of $\Pset=\operatorname{conv}\{G_1,
\ldots,G_K\}$, with $w_k\ge0$, $\sum_k w_k=1$. By concavity (Jensen's inequality
for a concave function),
\[
 \EVPI(P)=\EVPI\Bigl(\sum_k w_k G_k\Bigr)\ \ge\ \sum_k w_k\,\EVPI(G_k)
 \ \ge\ \min_{1\le k\le K}\EVPI(G_k),
\]
and the final bound is attained at the generator achieving the minimum. Hence
$\inf_{P\in\Pset}\EVPI(P)=\min_k\EVPI(G_k)$, with no optimization over $\Pset$
required. (For a general compact convex $\Pset$, concavity and continuity give
attainment of the infimum at an extreme point by Bauer's minimum principle, and
$\inf_{P}\EVPI(P)=\inf_{P\in\operatorname{ext}\Pset}\EVPI(P)$.)

(c) For the lower bound, fix $a\in\A$. Since $\E_P[\NB_a]\le g(P)$ for every $P$,
we have $\inf_{P}\E_P[\NB_a]\le\inf_{P}g(P)$. Taking the maximum over $a$ gives
$V_0^{\mathrm{LE}}=\max_a\inf_P\E_P[\NB_a]\le\inf_P g(P)$. Therefore
\[
 \EVPI_{\Pset}^{\mathrm{LE}}
 =\inf_P f(P)-V_0^{\mathrm{LE}}
 \ \ge\ \inf_P f(P)-\inf_P g(P).
\]
Because $f-g\le f-\inf_P g$ pointwise, $\inf_P(f-g)\le\inf_P f-\inf_P g$, so the
right-hand side is at least $\inf_P(f-g)=\underline{\EVPI}$. Combining,
$\EVPI_{\Pset}^{\mathrm{LE}}\ge\underline{\EVPI}$.

For the failure of the upper inequality we exhibit an explicit model. Let
$\A=\{0,1\}$, let $\Theta=\{1,2\}$, and set
\[
 \NB_0(1)=1,\ \NB_0(2)=-1,\qquad \NB_1(1)=-1,\ \NB_1(2)=1,
\]
with $B=1$. Take $\Pset$ to be the full simplex on $\Theta$, i.e.\
$\Pset=\operatorname{conv}\{e_1,e_2\}$, and write $x=P(\{1\})\in[0,1]$,
where $e_1,e_2$ are the point masses (vertices of the simplex) on states $1$
and $2$.
Then $M(\theta)=\max\{\NB_0,\NB_1\}=1$ for both states, so $f(P)=\E_P[M]=1$ for
every $P$, while $\E_P[\NB_0]=2x-1$ and $\E_P[\NB_1]=1-2x$, so
$g(P)=|2x-1|$ and $\EVPI(P)=1-|2x-1|$. The envelope is therefore
$[\,\underline{\EVPI},\overline{\EVPI}\,]=[0,1]$, with the maximum at $x=\tfrac12$
and the minimum at $x\in\{0,1\}$. The $\Gamma$-maximin quantities are
$V_\Theta^{\mathrm{LE}}=\inf_P f=1$ and
$V_0^{\mathrm{LE}}=\max\{\inf_x(2x-1),\inf_x(1-2x)\}=\max\{-1,-1\}=-1$, giving
\[
 \EVPI_{\Pset}^{\mathrm{LE}}=1-(-1)=2\ >\ 1=\overline{\EVPI}.
\]
Thus $\EVPI_{\Pset}^{\mathrm{LE}}$ exceeds the entire fixed-measure envelope.
\end{proof}

\begin{remark}[Why the overshoot occurs and how to read it]
\label{rem:overshoot}
The mechanism is a duality gap between actions and measures. Writing
$\psi(a,P)=\E_P[\NB_a]$, the $\Gamma$-maximin current value is the lower game
value $V_0^{\mathrm{LE}}=\max_a\inf_P\psi(a,P)$, while the classical current
value averaged worst-case is $\inf_P g(P)=\inf_P\max_a\psi(a,P)$. The minimax
inequality gives $V_0^{\mathrm{LE}}\le\inf_P g$, with equality only when some
action attains the lower envelope value $\inf_P\max_a\psi(a,P)$, and otherwise an
action--measure gap remains. The
envelope evaluates the perfect-information term $f$ and the current value $g$ at
the \emph{same} measure, whereas $\EVPI_{\Pset}^{\mathrm{LE}}$ pairs $\inf_P f$
with the lower game value $V_0^{\mathrm{LE}}$, and these are attained at
different measures. The lower-expectation EVPI is a coherent decision-analytic
value, the guaranteed gain from perfect information under a maximin attitude, but
is not a single-measure EVPI, and Theorem~\ref{thm:main}(c) shows this value can
exceed the largest single-measure EVPI. Operationally, this is a caution against
the common shortcut of reading a $\Gamma$-maximin information value as
the upper end of the sensitivity range, or computing it by maximizing classical
EVPI over the credal set. The two answers differ, and not merely by a sign or a
constant. The symmetric instance above is extremal in a concrete sense. The
perfect action label reverses between the two states, so the action--measure gap
is maximal and resolving uncertainty is most valuable at the center
($x=\tfrac12$) while the maximin current value is pinned to the worst state for
each action.

The strict overshoot also depends on the deterministic-action convention of
Section~\ref{sec:decision-rule-ip}. If randomized mixtures of the two actions were admitted, the even mixture would guarantee expected net benefit $0$ for every $P$,
raising the maximin current value from $-1$ to $0$ and the lower-expectation EVPI
to $1=\overline{\EVPI}$. More generally, allowing randomization makes the maximin
current value equal to $\inf_P g(P)$ (by the minimax theorem, the criterion being
bilinear in the mixing weights and in $P$), so
$\EVPI_{\Pset}^{\mathrm{LE}}=\inf_P f-\inf_P g\le\overline{\EVPI}$ and the strict overshoot disappears. The rule-specific value nevertheless remains a genuinely
distinct object even then because it pairs $\inf_P f$ with $\inf_P g$, attained at
different measures, rather than evaluating any single-measure EVPI. We therefore
state part~(c) for the deterministic, named-strategy action sets standard in
health-economic decision models. An analyst who can implement a randomized
strategy may add it to $\A$, with the understanding that doing so changes the
rule-specific value.
\end{remark}

\begin{remark}[The upper envelope endpoint: concavity and a linear program]
\label{rem:upper-lp}
By Theorem~\ref{thm:main}(a) the EVPI functional is concave, so its
\emph{maximum} over $\Pset$ is generally an interior point and is not delivered
by vertex enumeration. Only the lower endpoint enjoys the exact finite
computation in part~(b). For a finitely generated
$\Pset=\operatorname{conv}\{G_1,\ldots,G_K\}$ the upper endpoint is nonetheless
exactly computable by a finite linear program. Writing $P_w=\sum_k w_k G_k$ and
$D_{k,a}=\E_{G_k}[M]-\E_{G_k}[\NB_a]$, representation~\eqref{eq:min-affine} gives
$\EVPI(P_w)=\min_{a\in\A}\sum_k w_k D_{k,a}$, so
\[
\begin{aligned}
 \overline{\EVPI}
 &=\max_{w\in\Delta_K}\min_{a\in\A}\sum_{k=1}^K w_k D_{k,a}\\
 &=\max\Bigl\{\,t:\ t\le\textstyle\sum_k w_k D_{k,a}\ (a\in\A),\
 \sum_k w_k=1,\ w\ge0\,\Bigr\},
\end{aligned}
\]
a linear program in $(w,t)$ over the simplex $\Delta_K$. The generator maximum
$\max_k\EVPI(G_k)=\max_k\min_a D_{k,a}$ is only a \emph{lower} bound on this
optimum, and the gap can be substantial. In the applied example of
Section~\ref{sec:application}, the generator maximum is $\pounds 349$ but the
upper endpoint is $\pounds 514$. This asymmetry is the opposite of the linear
perfect-information value $f(P)=\E_P[M]$, whose endpoints are both attained at
extreme points because $f$ is affine. The current-value functional $g$ is convex,
so its upper endpoint is at a vertex while its lower endpoint may be interior.
These facts together explain which envelope endpoints can be read off the
generators and which require the linear program or an interior search.
\end{remark}

\begin{remark}[The lower-expectation EVPI components are also generator-exact]
On a finitely generated credal set the $\Gamma$-maximin quantities themselves
require no interior optimization, since $M$ is a fixed function and each
$\E_P[\NB_a]$ is affine in $P$, the infimum of an affine functional over
$\operatorname{conv}\{G_1,\ldots,G_K\}$ is attained at a generator, so
$V_\Theta^{\mathrm{LE}}=\min_k\E_{G_k}[M]$ and
$V_0^{\mathrm{LE}}=\max_a\min_k\E_{G_k}[\NB_a]$, whence
$\EVPI_{\Pset}^{\mathrm{LE}}=\min_k\E_{G_k}[M]-\max_a\min_k\E_{G_k}[\NB_a]$ is
computed exactly from the $K$ generators. This is consistent with the overshoot
of Theorem~\ref{thm:main}(c). The two minima over $k$ are generally attained at
different generators, which is precisely why $\EVPI_{\Pset}^{\mathrm{LE}}$ need
not coincide with any single $\EVPI(G_k)$ or lie within the envelope.
\end{remark}

\subsection{A sharp continuity bound and an a priori envelope width}
\label{sec:sharp-tv-sub}

The next result sharpens the total-variation continuity of the EVPI functional.
The improvement replaces the crude constant $4B$ by one built from the
oscillations of the net-benefit functions, $\operatorname{osc}(h)=\sup_\theta
h(\theta)-\inf_\theta h(\theta)$, and we show this constant is the best possible of
its form. Total variation is taken as $\|P-Q\|_{\mathrm{TV}}=\sup_{A}|P(A)-Q(A)|
\in[0,1]$.

\begin{proposition}[Sharp Lipschitz constant and envelope width]
\label{prop:sharp-tv}
For all $P,Q$,
\[
 |\EVPI(P)-\EVPI(Q)|\ \le\ L\,\|P-Q\|_{\mathrm{TV}},
 \qquad
 L=\max_{a\in\A}\operatorname{osc}(M-\NB_a),
\]
with $\operatorname{osc}$ as above. The constant satisfies
$L\le\operatorname{osc}(M)+\max_a\operatorname{osc}(\NB_a)\le 4B$ and cannot in
general be reduced. Consequently the EVPI envelope width obeys the a priori bound
$\overline{\EVPI}-\underline{\EVPI}\le L\,\operatorname{diam}_{\mathrm{TV}}
(\Pset)$, where $\operatorname{diam}_{\mathrm{TV}}(\Pset)=\sup_{P,Q\in\Pset}
\|P-Q\|_{\mathrm{TV}}$.
\end{proposition}

\begin{proof}
For any bounded measurable $h$ and the zero-sum signed measure $\nu=P-Q$, and any
constant $c$, $\int h\,d\nu=\int(h-c)\,d\nu$ because $\nu(\Theta)=0$. Hence
$|\int h\,d\nu|\le\sup_\theta|h-c|\cdot|\nu|(\Theta)$. The total variation of
$\nu$ as a signed measure is $|\nu|(\Theta)=2\|P-Q\|_{\mathrm{TV}}$, and choosing
$c=\tfrac12(\sup h+\inf h)$ minimizes $\sup_\theta|h-c|$ to $\tfrac12
\operatorname{osc}(h)$. Therefore
\begin{equation}
 |\E_P h-\E_Q h|\ \le\ \operatorname{osc}(h)\,\|P-Q\|_{\mathrm{TV}}.
 \label{eq:osc-bound}
\end{equation}
By the representation \eqref{eq:min-affine}, $\EVPI(P)=\min_a\E_P[D_a]$ with
$D_a=M-\NB_a$. Using $|\min_a u_a-\min_a v_a|\le\max_a|u_a-v_a|$ and then
\eqref{eq:osc-bound} with $h=D_a$,
\[
\begin{aligned}
 |\EVPI(P)-\EVPI(Q)|
 &=\Bigl|\min_a\E_P[D_a]-\min_a\E_Q[D_a]\Bigr|\\
 &\le\max_a\bigl|\E_P[D_a]-\E_Q[D_a]\bigr|\\
 &\le\max_a\operatorname{osc}(D_a)\,\|P-Q\|_{\mathrm{TV}},
\end{aligned}
\]
which is the stated bound with $L=\max_a\operatorname{osc}(M-\NB_a)$. Since
$\operatorname{osc}(M-\NB_a)\le\operatorname{osc}(M)+\operatorname{osc}(\NB_a)\le
\operatorname{osc}(M)+\max_b\operatorname{osc}(\NB_b)$ and each oscillation is at
most $2B$, we recover $L\le\operatorname{osc}(M)+\max_a\operatorname{osc}(\NB_a)
\le4B$. The constant $L$ is therefore never worse, and is typically strictly
smaller. The envelope-width corollary follows by taking the supremum over
$P,Q\in\Pset$, since $\overline{\EVPI}-\underline{\EVPI}=\sup_{P,Q}(\EVPI(P)-
\EVPI(Q))$.

For sharpness, use the instance of Theorem~\ref{thm:main}(c) with $\NB_0=(1,-1)$,
$\NB_1=(-1,1)$ and $M\equiv1$, so $D_0=M-\NB_0=(0,2)$ and $D_1=(2,0)$, giving
$\operatorname{osc}(D_0)=\operatorname{osc}(D_1)=2$ and $L=2$. Taking
$P=e_1$ and $Q=(\tfrac12,\tfrac12)$ gives $\|P-Q\|_{\mathrm{TV}}=\tfrac12$,
$\EVPI(P)=0$, $\EVPI(Q)=1$, and $|\EVPI(P)-\EVPI(Q)|=1=L\|P-Q\|_{\mathrm{TV}}$.
The bound is attained, so the factor $L$ multiplying $\|P-Q\|_{\mathrm{TV}}$
cannot be reduced uniformly: there exist models and pairs of measures for which
equality holds. (For a particular model the true Lipschitz constant may be
smaller. What is shown is that the formula is sharp across models.)
\end{proof}

\subsection{Vertex attainment fails for partial and sample information}
\label{sec:vertex-fail}

For EVPI the affine/concave structure produced the exact lower endpoint of
Theorem~\ref{thm:main}(b). The partial- and sample-information functionals lack
this structure. Although the individual posterior expectation
$\E_P[\NB_a\mid Z]=\E_P[\NB_a\mathbf{1}\{Z\}]/P(Z)$ is a ratio of affine
functionals of $P$, the post-information value itself is convex. Writing
$\Delta_Z$ for the measurable deterministic policies $\delta:\mathcal Z\to\A$,
\[
 W_Z(P)=\E_P\!\bigl[\max_a\E_P(\NB_a\mid Z)\bigr]
 =\sup_{\delta\in\Delta_Z}\E_P[\NB_{\delta(Z)}]
\]
is a supremum of maps $P\mapsto\E_P[\NB_{\delta(Z)}]$ that are affine in $P$, and
is therefore convex. The current value $g(P)=\max_a\E_P[\NB_a]$ is also convex, so
$\VOI_Z=W_Z-g$ is in general a difference of convex functions and need not be
convex or concave. The next proposition records that vertex enumeration is
consequently not exact for EVPPI or EVSI, and isolates a condition that restores
upper-endpoint vertex attainment.

\begin{proposition}[Failure of vertex attainment, and a sufficient condition]
\label{prop:evppi-vertex}
Let $\VOI_Z(P)=\E_P\!\bigl[\max_a\E_P(\NB_a\mid Z)\bigr]-g(P)$ denote the
classical EVPPI (for $Z=\phi(\theta)$) or EVSI (for a future dataset $Z$).
\begin{enumerate}
\item[(a)] There exist finitely generated credal sets on which
$\sup_{P\in\Pset}\VOI_Z(P)$ is attained only at a non-extreme point, so the
upper envelope endpoint is not equal to $\max_k\VOI_Z(G_k)$.
\item[(b)] Suppose a single action $a^\ast$ is pre-information optimal for every
$P\in\Pset$, i.e.\ $g(P)=\E_P[\NB_{a^\ast}]$ on $\Pset$. Then $\VOI_Z$ is convex
on $\Pset$. Consequently, for a finitely generated credal set
$\Pset=\operatorname{conv}\{G_1,\ldots,G_K\}$ the upper envelope endpoint equals
$\max_k\VOI_Z(G_k)$. More generally, if $\VOI_Z$ is upper semicontinuous on the
compact convex set $\Pset$, the upper endpoint is attained at an extreme point. (A $Z$-marginal that
is fixed across $\Pset$ is a convenient special case, in which the
post-information value takes the explicit form $\sum_z P(Z{=}z)\max_a
\E_P(\NB_a\mid Z{=}z)$, but it is not needed for the conclusion.)
\end{enumerate}
\end{proposition}

\begin{proof}
(a) Take the partial-information structure with a partial parameter $z\in\{0,1\}$
observed and a nuisance state $s\in\{0,1\}$ unobserved, so $\Theta=\{(z,s)\}$ has
four points and $Z=z$. With net benefits chosen so that the conditionally optimal
action under $z$ depends on the within-$z$ weighting of $s$, no single action is
pre-information optimal throughout, so by the difference-of-convex structure above
$\VOI_Z$ need not be convex or concave and its supremum can be strictly interior.
The following explicit instance makes this concrete.
Set $\NB_0\equiv0$ and $\NB_1=(1,-2,3,-1)$ on the states
$(0,0),(0,1),(1,0),(1,1)$, fix the signal marginal at $P(z{=}0)=P(z{=}1)=\tfrac12$,
and let a single parameter $t\in[0,1]$ set the within-signal weight
$P(s{=}1\mid z)=t$ for both signals, so the credal set is the segment with
generators $t=0$ and $t=1$. Writing out the post-information and current values,
\[
\begin{aligned}
 \VOI_Z(t)={}&\tfrac12\max\{0,\,1-3t\}+\tfrac12\max\{0,\,3-4t\}\\
 &{}-\max\{0,\,2-\tfrac72 t\},
\end{aligned}
\]
which is $0$ at both generators $t=0$ and $t=1$ but rises to the interior maximum
$\tfrac{5}{14}$ at $t=\tfrac47$. Vertex enumeration evaluates only $t\in\{0,1\}$,
returns $0$, and misses the maximum, so the upper envelope endpoint is not
$\max_k\VOI_Z(G_k)$. The same instance yields an EVSI counterexample by taking the
study datum to be $Y=Z$ under the degenerate sampling kernel
$K(\mathrm dy\mid\theta)=\delta_{\phi(\theta)}(\mathrm dy)$, so that EVSI equals
EVPPI exactly. Adding arbitrarily small observation noise gives a non-degenerate
sampling design. Because the state, action, and signal spaces are finite, the
finitely many conditional expected net benefits vary continuously with the noise
level, so the strict gap between the interior value $\tfrac{5}{14}$ and the endpoint
value $0$ persists for small enough noise and the EVSI maximum stays interior. The
full piecewise-linear calculation is given in
Supplement Section~5.

(b) The post-information value can be written as
\[
W_Z(P)=\sup_{\delta\in\Delta_Z}\E_P[\NB_{\delta(Z)}],
\]
a supremum of affine functionals of $P$, and is therefore convex on $\Pset$ with
no restriction on the $Z$-marginal. Under $g(P)=\E_P[\NB_{a^\ast}]$, the current
value is affine, so $\VOI_Z=W_Z-g$ is convex. For a finitely generated
$\Pset=\operatorname{conv}\{G_1,\ldots,G_K\}$,
\[
\VOI_Z\!\left(\sum_k w_kG_k\right)
\le \sum_k w_k\VOI_Z(G_k)
\le \max_k\VOI_Z(G_k),
\]
so the upper endpoint equals $\max_k\VOI_Z(G_k)$ with no continuity requirement.

For a general compact convex $\Pset$, extreme-point attainment follows from the
maximum principle for convex functions provided $\VOI_Z$ is upper semicontinuous.
This holds, for example, for finite-signal or deliberately discretized EVPPI/EVSI
policy classes, where $W_Z$ is a maximum of finitely many continuous affine maps
and hence continuous. Without a common pre-information action, $g$ is itself a genuine
maximum, so $\VOI_Z$ is a difference of convex functions and may have an interior
supremum, as in part~(a).
\end{proof}

\begin{remark}
Proposition~\ref{prop:evppi-vertex} delimits the reach of the exact computation.
The finite generator search is exact for the EVPI lower endpoint
(Theorem~\ref{thm:main}(b)) and, under a common pre-information action, for the
EVPPI/EVSI upper envelope endpoint, but not for EVPPI/EVSI envelopes in general.
Even under that condition, the restoration is one-sided. The functional is then
\emph{convex}, so its maximum is at a generator but its minimum can be interior,
and the lower endpoint may still require an interior search. This is the formal
justification for the two-stage outer search of
Section~\ref{sec:implementation}. Where the structural conditions hold, evaluating the generators delivers the upper endpoint. Elsewhere the discretized envelope
is an inner approximation that must be refined near sensitive endpoints.
\end{remark}

\subsection{The sign of rule-specific VOI and the role of constant policies}
\label{sec:sign-voi}

The final structural result concerns the \emph{sign} of the rule-specific value
of information. The possibility that information can have low or negative value
outside expected-utility single-prior theory is well known. Wakker showed how
violations of independence can generate information aversion, Schlee showed that
perfect information remains nonnegative under a weak dominance condition, and
dilation shows how conditioning can increase imprecision
\citep{wakker1988nonexpected,schlee1991perfect,seidenfeld1993dilation}. The result below is not a new general theorem on information aversion. It is the
lower-expectation, normal-form statement needed for the present VOI framework.
Classically, perfect or sample information can never have a negative value because
an agent who can ignore what is learned is never worse off for having learned it.
Under imprecision this guarantee becomes conditional, and the key sufficient
condition is the availability of policies that ignore the signal. Without it,
nonnegativity is no longer guaranteed. Recall from Section~\ref{sec:framework}
that a policy $\delta\in\Delta_Z$ maps the signal $Z$ to an action, that a
\emph{constant} policy chooses the same action for every value of $Z$, and that
the lower-expectation value of information is
$\VOI_Z^{\mathrm{LE}}=V_Z^{\mathrm{LE}}-V_0^{\mathrm{LE}}$, with
$V_Z^{\mathrm{LE}}=\sup_{\delta\in\Delta_Z}\inf_{P\in\Pset_Z}\E_P[\NB_{\delta(Z)}]$
and $V_0^{\mathrm{LE}}=\max_{a\in\A}\inf_{P\in\Pset}\E_P[\NB_a]$, where
$\Pset_Z$ extends $\Pset$ in the sense that its $\theta$-marginal set is $\Pset$.

\begin{theorem}[Sign of rule-specific VOI]
\label{thm:sign}
Let $Z$ be information observed before the final action is chosen.
\begin{enumerate}
\item[(a)] \textnormal{(Nonnegativity with constant policies.)} If $\Delta_Z$
contains all constant policies, then $\VOI_Z^{\mathrm{LE}}\ge0$.
\item[(b)] \textnormal{(Negativity without them.)} If $\Delta_Z$ excludes the
constant policies, so the agent must act on the signal, then
$\VOI_Z^{\mathrm{LE}}$ can be strictly negative, so cost-free information can be
``bad''. Explicitly, let $\A=\{0,1\}$ and $\theta=(z,s)\in\{0,1\}^2$ with $z$ the
observed signal and $s$ an unobserved nuisance state, ordered
$(0,0),(0,1),(1,0),(1,1)$, and take
\[
 \NB_0=(0,0,0,0),\qquad \NB_1=(1,-1,-1,1),
\]
with $\Pset_Z=\operatorname{conv}\{G_1,G_2\}$, $G_1=(0,\tfrac12,\tfrac12,0)$,
$G_2=(\tfrac12,0,0,\tfrac12)$. Then $V_0^{\mathrm{LE}}=0$, every signal-responsive
policy has guaranteed value $-\tfrac12$, so $\VOI_Z^{\mathrm{LE}}=-\tfrac12<0$
over the restricted class $\{\delta:\delta(0)\ne\delta(1)\}$, whereas
$\VOI_Z^{\mathrm{LE}}=0$ over the full class containing the constant policies.
\end{enumerate}
\end{theorem}

\begin{proof}
(a) A constant policy chooses the same action $a$ regardless of $Z$. Because the
$\theta$-marginals of $\Pset_Z$ are exactly $\Pset$, the lower expected value of
the constant policy ``always $a$'' computed over $\Pset_Z$ equals
$\inf_{P\in\Pset}\E_P[\NB_a]$, its current guaranteed value. Hence every action
available before observing $Z$ is available afterwards as a constant policy with
the same guaranteed value, so $V_Z^{\mathrm{LE}}=\sup_{\delta\in\Delta_Z}
\inf_{P}\E_P[\NB_{\delta(Z)}]\ge\max_{a\in\A}\inf_{P\in\Pset}\E_P[\NB_a]
=V_0^{\mathrm{LE}}$, giving $\VOI_Z^{\mathrm{LE}}\ge0$.

(b) Write a measure as $(q_{00},q_{01},q_{10},q_{11})$. For the constant
policies, $\E_P[\NB_0]=0$ for all $P$, so $\inf_{P\in\Pset_Z}\E_P[\NB_0]=0$, while
$\E_{G_1}[\NB_1]=\tfrac12(-1)+\tfrac12(-1)=-1$ gives $\inf_P\E_P[\NB_1]\le-1$.
Hence $V_0^{\mathrm{LE}}=\max\{0,\inf_P\E_P[\NB_1]\}=0$, attained by action $0$.
For the signal-responsive policy $\delta=(0,1)$ the realized net benefit is
$\NB_0$ on $z=0$ (states $(0,0),(0,1)$) and $\NB_1$ on $z=1$ (states
$(1,0),(1,1)$), i.e.\ the gamble $(0,0,-1,1)$, with lower expectation
$\min\{G_1\cdot(0,0,-1,1),\,G_2\cdot(0,0,-1,1)\}=\min\{-\tfrac12,\tfrac12\}
=-\tfrac12$. By the $z\leftrightarrow s$ symmetry the policy $(1,0)$ also has
lower expectation $-\tfrac12$. Thus the best guaranteed value over the restricted
class is $-\tfrac12$ and $\VOI_Z^{\mathrm{LE}}=-\tfrac12-0=-\tfrac12<0$.
Including the constant policies restores the maximum to $0$ and $\VOI_Z^{\mathrm{LE}}=0$.
\end{proof}

\begin{remark}
Theorem~\ref{thm:sign} is the single-decision, static counterpart of the sequential
phenomenon identified by Bradley and Steele \citep{bradley2016free}. Under a
$\Gamma$-maximin attitude, cost-free information can be ``bad'' when the agent cannot fall back on a
policy that ignores the signal. Here the safeguard is
explicit and elementary, namely the presence of the constant policies in
$\Delta_Z$. Part~(a) shows this condition is sufficient for nonnegativity, while part~(b)
shows that without the constant policies nonnegativity is no longer guaranteed.
Because the argument uses only the availability of constant policies, part~(a)
holds for any signal $Z$, hence for EVPPI ($Z=\phi(\theta)$) and EVSI ($Z=Y_d$).
In the credal-set lower-expectation setting it thus mirrors, for partial and
sample information, the perfect-information nonnegativity of
\citet{schlee1991perfect}, with the constant policies playing the role of the weak
dominance condition. The result is therefore consistent with the broader literature on information
aversion, dilation, and sequential decisions under lower previsions
\citep{wakker1988nonexpected,schlee1991perfect,seidenfeld1993dilation,bradley2016free,seidenfeld2004contrast,troffaes2007decision,augustin2001ambiguous,huntley2012normal}. In the
research-prioritization use of VOI the analyst is always free not to act on
inconclusive evidence, so the constant policies are present and
$\VOI_Z^{\mathrm{LE}}\ge0$ holds.
\end{remark}

\section{A unified estimation framework}
\label{sec:implementation}

The computations below use a common template with an \emph{outer}
representation of the admissible set wrapped around an \emph{inner}
single-measure Monte Carlo evaluation. The outer layer may enumerate admissible
measures directly, list generators of a convex credal set, or discretize
continuous components. The inner layer is run once per represented measure or
generator.

The distinction between estimand and estimator is explicit. The EVPI endpoint
identities in Theorem~\ref{thm:main} and Remark~\ref{rem:upper-lp} are exact
statements about population functionals. The reported numerical values are Monte
Carlo estimates of those functionals. The EVPPI and EVSI calculations used in the
worked application are validated single-measure approximations, not exact analytic
EVPPI/EVSI evaluations. Their role is to illustrate how standard VOI estimators
can be wrapped in an outer credal-set analysis without confusing numerical
approximation with the population estimands.

Represent the outer layer by $\Hset=\{h_1,\ldots,h_R\}$, where $h_m$ labels a
simulable measure $P_{h_m}$ or a generator. For each $h_m$, draw
$\theta_m^{(i)}\sim P_{h_m}$ and store
\begin{equation}
 S_{m,i,a}=\NB_a\!\left(\theta_m^{(i)}\right),\qquad a\in\A,
 \label{eq:inner-substrate}
\end{equation}
with $z_m^{(i)}=\phi(\theta_m^{(i)})$ for EVPPI or a simulated future dataset
$Y_d^{(i)}$ for EVSI. The empirical lower-expectation current value is
\begin{equation}
 \widehat V_0^{\mathrm{LE}}=\max_{a\in\A}\ \min_{1\le m\le R}\
 \frac1N\sum_{i=1}^N S_{m,i,a}.
 \label{eq:hat-current}
\end{equation}
For EVPI, the common pointwise optimal policy gives
\begin{equation}
 \widehat V_{\Theta}^{\mathrm{LE}}
 =\min_{1\le m\le R}\ \frac1N\sum_{i=1}^{N}\max_{a\in\A} S_{m,i,a},
 \qquad
 \widehat{\EVPI}{}^{\mathrm{LE}}_{\Pset}=\widehat V_{\Theta}^{\mathrm{LE}}-
 \widehat V_0^{\mathrm{LE}}.
 \label{eq:hat-evpi}
\end{equation}
For EVPPI or EVSI under the lower-expectation rule, one searches over a policy class.
\begin{equation}
 \widehat V_Z^{\mathrm{LE}}
 =\max_{\delta\in\Delta_Z}\ \min_{1\le m\le R}\
 \frac1N\sum_{i=1}^{N} S_{m,i,\delta(z_m^{(i)})},
 \qquad
 \widehat{\VOI}{}^{\mathrm{LE}}_{Z}=\widehat V_Z^{\mathrm{LE}}-\widehat V_0^{\mathrm{LE}}.
 \label{eq:hat-policy}
\end{equation}
This search is finite only when $Z$ has few states. For continuous or finely
discretized signals, one must use a structured policy class, dynamic programming,
regression-based decision rule, or another controlled approximation. Bayes-optimal
policies under individual admissible measures are only a heuristic screen and do
not generally contain the maximin-optimal policy.

For a conventional fixed-measure target $T(P)$, a validated single-measure
estimator can be applied at each represented measure,
\begin{equation}
 \bigl[\widehat{\underline T},\widehat{\overline T}\bigr]
 =\Bigl[\min_m\widehat T(h_m),\ \max_m\widehat T(h_m)\Bigr].
 \label{eq:hat-envelope}
\end{equation}
This is exact when $\Hset$ enumerates the admissible measures. If $\Hset$ lists
generators of a convex credal set, generator extrema are exact only when justified
by structure. For EVPI, the lower endpoint is generator-exact
(Theorem~\ref{thm:main}(b)) and the upper endpoint uses the LP of
Remark~\ref{rem:upper-lp}. EVPPI/EVSI generator ranges are generally inner
approximations unless Proposition~\ref{prop:evppi-vertex}(b) applies. A grid over a
continuous component is also an inner approximation and can understate the envelope
width, so threshold-sensitive cases require refinement.

In the application, single-measure EVSI is estimated by a regression-based
posterior-mean rule with common random numbers under a normal--normal update of the
treatment effect. This is a finite Monte Carlo regression approximation, not an
exact nested posterior estimator. Details, ordering checks, and variance diagnostics
are in the supplementary material.

\begin{workflowalgorithm}{Unified IP-VOI computation (lower-expectation rule shown, with \texorpdfstring{$\rho$}{rho} substituting in the inner aggregation).}
\label{alg:unified}
 \item \textbf{Admissible set.} Build $\Hset=\{h_1,\ldots,h_R\}$ representing $\Pset$ by measures, generators, or a discretized grid, recording the reference index $m_0$ if used.
 \item \textbf{Decision model and target.} Specify $\A$, $\Theta$, $\NB_a(\theta)$, and $Z$ ($\theta$ for EVPI, $\phi(\theta)$ for EVPPI, $Y_d$ for EVSI).
 \item \textbf{Inner simulation.} For each represented measure or generator, draw $\theta_m^{(i)}$, form $S_{m,i,a}$ via \eqref{eq:inner-substrate}, and store $z_m^{(i)}$ for EVPPI/EVSI.
 \item \textbf{Estimands.} Compute rule-specific EVPI by \eqref{eq:hat-evpi}, rule-specific EVPPI/EVSI by \eqref{eq:hat-policy} when the policy class is finite or controlled, and fixed-measure envelopes by validated single-measure estimators, using structural endpoint calculations for generator-defined convex credal sets when needed.
 \item \textbf{Stability.} Report reference and endpoint values, endpoint-defining assumptions, threshold classification, and Monte Carlo error.
\end{workflowalgorithm}

\section{Worked application}
\label{sec:application}

\subsection{A decision model with a credal set of inputs}

To illustrate the framework on a concrete credal set, we use the Heath--Baio
chemotherapy adverse-event decision model, a standard value-of-information testbed
\citep{strong2014evppi,heath2020evsi,kunst2020evsi}. The Collaborative Network for
Value of Information distributes a public reference implementation \citep{voi2025}.
At a willingness to pay of $\pounds 20{,}000$ per quality-adjusted life year, the
published single-measure analysis reports an EVPPI of about $\pounds 333$ and an
EVSI of about $\pounds 242$ for a two-arm trial of $150$ patients per arm. We use
these as numerical anchors rather than identical estimands. Our validation targets the
treatment-effect EVPPI for $\ell$, whereas the published EVPPI is for adverse-event
probabilities, and the calibration is not a bit-for-bit re-execution of the
reference software. The example places a conventional reference-measure VOI inside
a credal-set envelope and exhibits both decision instability across admissible
measures and the structural endpoint phenomena proved above.

\subsection{Decision structure and admissible probability measures}
\label{sec:case-structure}

The model compares two strategies for managing chemotherapy adverse events.
Strategy $a_0$ is standard care. Strategy $a_1$ is a side-effect-reducing
intervention that lowers the rate at which managed (ambulatory) adverse events
progress to inpatient (hospital) care, at a higher per-cycle drug cost. A four-state Markov model
(ambulatory care, hospital care, recovery, death) is run over a short adverse-event
horizon. Each cycle accrues state-specific utility and cost, and the drug cost of
the active strategy is added while the patient is alive. Net benefit for strategy
$a$ is
\begin{equation}
 \NB_a=\lambda\, Q_a(\theta)-C_a(\theta),\qquad \lambda=20{,}000,
 \label{eq:case-nb}
\end{equation}
with $Q_a$ the accrued quality-adjusted life-years over the horizon and
$C_a$ the accrued cost. Larger values are preferred. The intervention acts on the
ambulatory-to-hospital transition through a log odds ratio
$\ell\in\mathbb R$, $p^{\mathrm{trt}}=\operatorname{logit}^{-1}
\{\operatorname{logit}(p^{\mathrm{soc}})+\ell\}$, so a more negative $\ell$ is a
larger clinical benefit.

The reference probability measure $P_0$ used four choices. These were a pooled
treatment-effect synthesis for $\ell$, a registry source for the baseline
ambulatory-to-hospital risk, a pessimistic source for the hospital mortality
risk, and zero correlation between the efficacy and harm parameters. The
admissible set $\Pset$ crossed four sources of imprecision. These were three
treatment-effect evidence specifications for $\ell$ (trial-only, pooled, and a
conservative down-weighting), two baseline-risk sources, two hospital-mortality
sources, and a correlation between the efficacy and harm parameters ranging over
$[-0.4,0.4]$. The three discrete dimensions generate $3\times2\times2=12$
generating measures $G_1,\ldots,G_{12}$ at zero correlation, whose convex hull we
denote
\[
 \mathcal{C}_0=\operatorname{conv}\{G_1,\ldots,G_{12}\}
 \qquad(\text{efficacy--harm correlation fixed at }0).
\]
This finitely generated set is the domain of the EVPI envelope and threshold
analysis below, and is distinct from the reference \emph{measure} $P_0$, which is
itself one of the generators. Treating $\mathcal{C}_0$ as a credal set means
adopting the coherent convex closure of the 12 generator measures. The convex
combinations $\sum_k w_k G_k$ are included as admissible precise measures in the
closed uncertainty model, and no single vector of weights is selected or
interpreted as a defended model average, so the admissible set is convex by
construction. Were the 12 measures instead treated as a nonconvex family of
distinct scenarios, the upper end of the reported EVPI range would be the generator
maximum (about $\pounds 349$) rather than the convex-hull value (about
$\pounds 514$). The analysis below is therefore explicitly a convex-credal-set
analysis, and we report the generator maximum separately so the distinction is
visible. The efficacy--harm correlation adds a continuous
dimension, explored separately on a grid as a sensitivity analysis rather than
folded into the convex-hull optimization. Table~\ref{tab:case-admissible-set}
summarizes the construction.

\begin{table}[H]
\centering
\caption{Construction of the admissible family $\Pset$ for the chemotherapy example.
Each row is one source of imprecision and the full admissible family is their joint range. Here $\ell$ is
the treatment-effect log-odds-ratio for the ambulatory-to-hospital transition. The
three discrete dimensions give $3\times2\times2=12$ generating measures
$G_1,\ldots,G_{12}$ whose convex hull is $\mathcal{C}_0$ (the set used for
Tables~\ref{tab:chemo-evpi}--\ref{tab:chemo-thresholds}, with correlation fixed at
zero), while the continuous
efficacy--harm correlation is handled separately on a grid as a sensitivity
analysis (an inner approximation, refined where it matters).}
\label{tab:case-admissible-set}
\small
\resizebox{\textwidth}{!}{%
\begin{tabular}{@{}llc@{}}
\toprule
Source of imprecision & Admissible values & Representation \\
\midrule
Treatment effect $\ell$ & trial-only, pooled, conservative & $3$ choices \\
Baseline ambulatory$\to$hospital risk & registry, trial & $2$ choices \\
Hospital mortality risk & optimistic, pessimistic & $2$ choices \\
Efficacy--harm correlation & interval $[-0.4,0.4]$ & continuous grid \\
\midrule
\multicolumn{2}{@{}l}{Zero-correlation generator measures} & $12$ \\
\bottomrule
\end{tabular}%
}
\end{table}

Each member of the admissible family is a complete probability measure under which
the model can be simulated and conventional EVPI, EVPPI, and EVSI estimated. Tables
\ref{tab:chemo-evpi} and \ref{tab:chemo-thresholds} report exact EVPI endpoints over
$\mathcal{C}_0$, the correlation-zero convex hull. Because $\mathcal{C}_0$ is finitely generated and EVPI has the structure established in Theorem~\ref{thm:main}, its
endpoints are computed exactly apart from Monte Carlo error, the lower endpoint by
generator enumeration (Theorem~\ref{thm:main}(b)) and the upper endpoint by the
linear program of Remark~\ref{rem:upper-lp}. The continuous correlation dimension is
explored separately by grid refinement, so its contribution to the envelope is an
inner approximation that we refine where it matters (Section~\ref{sec:case-evpi-evppi}).

\subsection{EVPI, EVPPI, and a decision that is not stable across the set}
\label{sec:case-evpi-evppi}

Under the reference measure $P_0$, the side-effect-reducing intervention $a_1$
had the higher expected net benefit, by about $\pounds 150$ per patient.
Conventional EVPI under $P_0$ was approximately $\pounds 329$ per patient (Monte
Carlo standard error about $\pounds 4$ from a sample of $2\times10^4$ draws),
and EVPPI for the treatment-effect parameter was approximately $\pounds 329$ as
well, essentially the whole of the EVPI, confirming that the treatment effect
dominates decision uncertainty under the reference measure. Over the
correlation-zero convex hull $\mathcal{C}_0$, the fixed-measure EVPI
envelope ran from about $\pounds 84$ to about $\pounds 514$ per patient
(Table~\ref{tab:chemo-evpi}). The lower endpoint is attained at a generator and is
computed exactly by vertex enumeration (Theorem~\ref{thm:main}(b)). The upper
endpoint is attained at an \emph{interior} mixture of two generators and is
computed by the linear program of Remark~\ref{rem:upper-lp}, with the largest
single-generator EVPI ($\pounds 349$) understating it because the concave EVPI
functional need not be maximized at a vertex. The reference value lies in the
interior of this envelope. Although $a_1$ is preferred under $P_0$, that preference
is not robust across $\mathcal{C}_0$, and the conventional EVPI varies approximately
sixfold between the most and least favorable points of $\mathcal{C}_0$.

\begin{table}[H]
\centering
\caption{Conventional EVPI per patient under the reference measure $P_0$, at the
generator extremes, and at the convex-hull envelope endpoints over $\mathcal{C}_0$, the
correlation-zero convex hull (Monte Carlo standard
errors of about $\pounds 2$--$4$). The lower endpoint is attained at a generator
(vertex enumeration is exact for it). The upper endpoint is attained at an interior
mixture, $0.37\,G_{\mathrm A}+0.63\,G_{\mathrm B}$ with $G_{\mathrm A}=$
trial-only/registry/optimistic and $G_{\mathrm B}=$ conservative/registry/optimistic,
obtained from the linear program of Remark~\ref{rem:upper-lp}. The largest single
generator EVPI ($\pounds 349$) understates it, and at that mixture the two
strategies tie in expected net benefit. Generator identities among near-equal EVPI
values are Monte Carlo--sensitive. The lower endpoint ($\pounds 84$) and a second
trial-only generator differ by under a pound, and the two largest generators
($\approx\pounds 348$--$349$) carry opposite preferred strategies, so the envelope
\emph{values} are the primary estimands rather than the corner labels. The decision
flip itself rests on the clear separation between the trial-only generators (which
prefer $a_1$) and the conservative generators (which prefer $a_0$), so the
treatment decision, not merely the magnitude of VOI, is unstable across $\mathcal{C}_0$.}
\label{tab:chemo-evpi}
\small
\resizebox{\textwidth}{!}{%
\begin{tabular}{@{}llcr@{}}
\toprule
Measure & Effect / baseline / mortality & Preferred & EVPI \\
\midrule
Reference $P_0$ (generator) & pooled / registry / pessimistic & $a_1$ & $\pounds 329$ \\
Envelope lower endpoint (vertex) & trial-only / registry / pessimistic & $a_1$ & $\pounds 84$ \\
Largest single generator & pooled / trial / pessimistic & $a_0$ (near tie) & $\pounds 349$ \\
Envelope upper endpoint (interior) & $0.37\,G_{\mathrm A}+0.63\,G_{\mathrm B}$ & tie & $\pounds 514$ \\
\bottomrule
\end{tabular}%
}
\end{table}

Unlike a setting in which one strategy dominates throughout, here the
\emph{preferred strategy changes within the admissible set}. The
side-effect-reducing intervention $a_1$ is preferred under the trial-only effect
specification (and under most pooled-effect generators), whereas standard care
$a_0$ is preferred under the conservative effect specification. The flip between
these two effect specifications is well separated and not an artifact of Monte
Carlo noise. (One pooled-effect corner, with the trial baseline and pessimistic
mortality, also tips to $a_0$, but only by a near-tie margin, so the decision flip
does not rest on it.)
The treatment recommendation, not merely the value of further research, is
therefore measure-dependent. This is exactly the situation in which reporting a
single reference EVPI is most misleading, and in which the threshold reading of
Section~\ref{sec:misleading-standard-voi} is most useful
(Table~\ref{tab:chemo-thresholds}).

\begin{table}[H]
\centering
\caption{Stability of the research conclusion across value thresholds. Here $\tau$
is the per-patient value threshold above which acquiring information is judged
worthwhile, $P_0$ is the reference measure (EVPI $\approx\pounds 329$), and the
fixed-measure EVPI envelope is $[\pounds 84,\pounds 514]$ over $\mathcal{C}_0$, the
correlation-zero convex hull of the 12 generators. The conclusion is robust when the whole envelope falls on one
side of $\tau$ and assumption-dependent when the envelope straddles it. A
single-measure report is most misleading in the middle row, where $P_0$ alone
gives a clean verdict that the admissible set does not support.}
\label{tab:chemo-thresholds}
\small
\begin{tabular}{@{}rccl@{}}
\toprule
Threshold $\tau$ & Using $P_0$ only & Envelope vs.\ $\tau$ & Conclusion over $\mathcal{C}_0$ \\
\midrule
$\pounds 50$ & worthwhile & entirely above & robust: worthwhile \\
$\pounds 250$ & worthwhile & straddles $\tau$ & assumption-dependent \\
$\pounds 800$ & not worthwhile & entirely below & robust: not worthwhile \\
\bottomrule
\end{tabular}
\end{table}

\paragraph{The EVPI lower endpoint is computed exactly from the generators}
Because $\mathcal{C}_0$ is finitely generated,
Theorem~\ref{thm:main}(b) applies, so the lower envelope endpoint equals the smallest
EVPI among the 12 generators, with no optimization over the interior. Evaluating
conventional EVPI at each generator and taking the minimum returns the
$\pounds 84$ endpoint of Table~\ref{tab:chemo-evpi} directly, illustrating the
exact finite computation.

\paragraph{The EVPI upper endpoint and the rule-specific value are distinct}
Table~\ref{tab:chemo-evpi} illustrates the main structural results. The EVPI
lower endpoint is the smallest generator value, as Theorem~\ref{thm:main}(b)
predicts. The upper endpoint is different. The largest single-generator EVPI is
$\pounds 349$, but the LP of Remark~\ref{rem:upper-lp} gives $\pounds 514$ at the
interior mixture $0.37G_{\mathrm A}+0.63G_{\mathrm B}$. Thus vertex enumeration is
exact for the lower endpoint but understates the upper endpoint by about a third.
The lower-expectation EVPI over $\mathcal C_0$ is another object again,
approximately $\pounds 189$. This is the gain in guaranteed decision value, not an
envelope endpoint. In this application it lies inside the fixed-measure envelope,
but Theorem~\ref{thm:main}(c) permits it to exceed the whole envelope. The minimal
instance in Table~\ref{box:overshoot} makes that separation explicit.

\begin{table}[H]
\centering
\caption{A minimal instance of the $\Gamma$-maximin overshoot
(Theorem~\ref{thm:main}(c)). Two strategies and two states with mirror-image net
benefits, and $\Pset$ the full set of measures on the two states, written
$x=P(\text{state }1)$. The rule-specific ($\Gamma$-maximin) EVPI equals $2$ and
lies strictly above the entire fixed-measure envelope $[0,1]$, because the
worst-case perfect-information value and the worst-case current value are attained
at different measures. The rule-specific value is therefore not any single-measure
EVPI. This inset uses the deterministic named-strategy convention. Randomized
mixtures are not included as actions, and admitting them would raise the maximin
current value and remove the strict overshoot in this two-state example
(Remark~\ref{rem:overshoot}).}
\label{box:overshoot}
\small
\begin{tabular}{@{}lcc@{}}
\toprule
 & State $1$ & State $2$ \\
\midrule
Net benefit of strategy $0$ & $+1$ & $-1$ \\
Net benefit of strategy $1$ & $-1$ & $+1$ \\
\midrule
Conventional EVPI$(x)=1-|2x-1|$ & \multicolumn{2}{c}{max $=1$ at $x=\tfrac12$} \\
Fixed-measure envelope $[\underline{\EVPI},\overline{\EVPI}]$ & \multicolumn{2}{c}{$[0,\,1]$} \\
Perfect-information value $V_\Theta^{\mathrm{LE}}=\inf_x\E_x[M]$ & \multicolumn{2}{c}{$1$} \\
Current value $V_0^{\mathrm{LE}}=\max_a\inf_x\E_x[\NB_a]$ & \multicolumn{2}{c}{$-1$} \\
Rule-specific $\EVPI_{\Pset}^{\mathrm{LE}}=V_\Theta^{\mathrm{LE}}-V_0^{\mathrm{LE}}$ & \multicolumn{2}{c}{$2\;(>1)$} \\
\bottomrule
\end{tabular}
\end{table}

\subsection{EVSI, the research-commissioning decision, and vertex enumeration}
\label{sec:case-evsi}

We evaluated a future two-arm trial that updates the treatment-effect parameter
$\ell$ under a normal--normal approximation. Under the reference measure, EVSI rose
from about $\pounds 97$ at $n=20$ per arm, to $\pounds 243$ at the $150$-per-arm
benchmark, to $\pounds 320$ at $n=2000$, approaching the treatment-effect
EVPPI/EVPI ceiling of about $\pounds 329$. Because the trial informs only $\ell$,
the nested-information ordering $\EVSI\le\EVPPI(\ell)\le\EVPI$ applies. Across the
12 generators, EVSI at $n=2000$ ranged from about $\pounds 78$ to $\pounds 339$.
By Proposition~\ref{prop:evppi-vertex}, this generator range is an inner
approximation to the EVSI envelope over the convex hull.

\paragraph{Vertex enumeration fails for the partial-information envelope}
The four-point construction of Proposition~\ref{prop:evppi-vertex} gives a minimal
two-action example in which the EVPPI is zero at both generator endpoints but has a
strict interior maximum of about $0.36$ net-benefit units. Vertex enumeration would
therefore miss the upper endpoint entirely. The chemotherapy correlation grid plays
a different role. The treatment-effect EVPPI varies only modestly across the
nine-point grid, from about $\pounds 326$ to $\pounds 333$, so it illustrates that
grid-based searches are inner approximations that should be refined when threshold
classifications are close, rather than itself exhibiting an interior maximum.

\subsection{Validation of the estimators}
\label{sec:validation}

Because IP-VOI reuses single-measure VOI estimators inside an outer search, the
framework is only as trustworthy as those estimators. We therefore ran a small
validation study. First, for a two-strategy model with a normally distributed
net-benefit difference, EVPI has a closed form. Across a grid of mean advantages,
the Monte Carlo EVPI matched the closed form to within Monte Carlo error
(maximum absolute discrepancy below $\pounds 1$ per patient across the grid, Figure~\ref{fig:validation}a). Second, the EVSI estimator reproduced its
theoretical limit on the reference measure, rising monotonically with the planned
sample size from about $\pounds 97$ at a small trial ($n=20$) toward the treatment-effect EVPPI ceiling of
about $\pounds 329$ as the trial grew large (Figure~\ref{fig:validation}b).
Third, across all 12 zero-correlation generators of $\mathcal{C}_0$ the estimates
respected the expected ordering $\EVSI\le\EVPPI\le\EVPI$ for this nested
information structure, within Monte Carlo error.
These checks indicate that the envelope widths reported above reflect imprecision
across the admissible set rather than estimator artifacts.

\begin{figure}[H]
\centering
\includegraphics[width=0.92\linewidth]{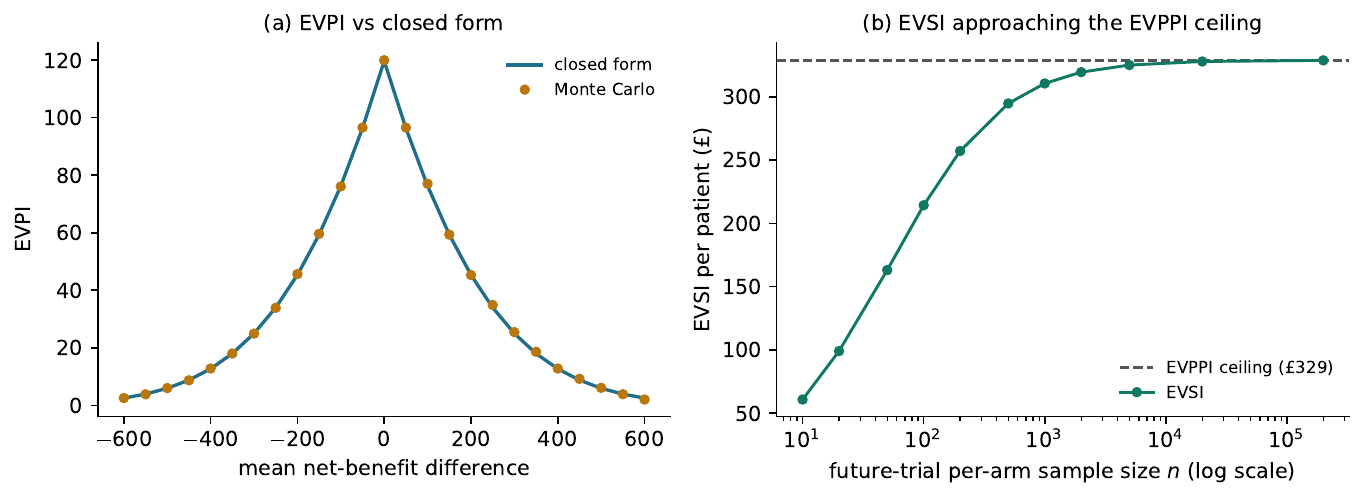}
\caption{Validation of the value-of-information estimators against known limits in
these examples. (a)~Monte Carlo EVPI (points) against the closed-form value (line)
for a two-strategy normal model. (b)~EVSI per patient for the reference measure as a
function of the future-trial per-arm sample size $n$ (log scale), approaching the
treatment-effect EVPPI ceiling (dashed) from below. This checks the regression-based
estimator in the present calibration rather than establishing a general guarantee.
EVPI is the expected value of
perfect information, EVPPI the expected value of partial perfect information, and
EVSI the expected value of sample information.}
\label{fig:validation}
\end{figure}

\section{Discussion}
\label{sec:discussion}

\subsection{What changes when the measure becomes a set}
Moving from a single measure to a credal set changes the estimand, not merely
the computation. Distributional and expectation bounds do not by themselves rank
actions, so a decision criterion for imprecision must be supplied before a
rule-specific VOI number exists. The framework keeps three objects distinct. The
rule-specific VOI values information under a stated decision criterion. The
fixed-measure envelope is a sensitivity diagnostic for how strongly a classical
VOI result depends on the selected precise measure. Probability bounds analysis
is a constructive route for specifying and propagating the credal set, not a
decision rule.

Two structural facts are practically important. First, rule-specific VOI under
the lower-expectation criterion is nonnegative whenever the policy class contains
constant policies, because uninformative observations can be ignored. Without
that fallback, negative values can occur in models of sequential choice under
imprecision \citep{bradley2016free}. Second, EVPI has a special structure because
the pointwise optimal act after perfect information is the same function of
\(\theta\) for every admissible measure. This gives the concavity and endpoint
results above. Partial and sample information do not share this simplification:
the relevant post-information policies depend on the signal and generally cannot
be reduced to the EVPI endpoint calculations.

\paragraph{Relation to the value-of-evidence literature}
The sign of rule-specific VOI connects directly to a known tension for the
imprecise probabilist. \citet{bradley2016free} show that, under a
\(\Gamma\)-maximin attitude with a generalized conditionalization rule, a
sequential agent can be rationally required to pay to avoid cost-free evidence.
Free evidence can be bad. E-admissibility and maximality are more hospitable to
the principle that evidence is never bad
\citep{seidenfeld2004contrast,troffaes2007decision}. Our results are consistent
with this picture in the single-node setting. When the policy class contains the
constant policies, so that the agent may ignore the signal and re-optimize,
Theorem~\ref{thm:sign}(a) guarantees nonnegative lower-expectation VOI. This is
the benign regime. When constant policies are excluded, representing a commitment
to act on the signal, Theorem~\ref{thm:sign}(b) gives a finite credal set and
payoff structure for which the lower-expectation VOI is strictly negative. The
result is a static analogue of the Bradley--Steele phenomenon. The framework
therefore identifies the absence of an ignore-the-signal fallback as one mechanism
by which free information can be harmful under \(\Gamma\)-maximin, while the
research-prioritization use of VOI normally falls in the benign regime because
the analyst is free not to act on inconclusive evidence.

\subsection{Relationship to credal representations and decision rules}
The development is stated for a generic credal set and therefore specializes to
standard representations. Probability bounds analysis is natural when the
evidence is expressed as ranges, support bounds, moments, quantiles, CDF bounds,
or uncertain dependence. The same formulation admits coherent lower previsions,
credal sets from linear constraints, robust Bayesian classes of priors, interval
probabilities, Dempster--Shafer structures, sets of copulas, and moment-based
ambiguity sets from distributionally robust analysis
\citep{walley1991,augustin2014introduction,troffaes2007decision,weichselberger2000interval,berger1985statistical}.
The lower-expectation criterion used in the worked development is the
\(\Gamma\)-maximin functional of a coherent lower prevision. Replacing the
criterion changes the aggregation across measures but not the simulation
substrate. Scalar criteria such as minimax regret and weighted lower--upper rules
can therefore be obtained by substituting the corresponding criterion
\(C_\rho\). Set-valued criteria such as maximality, interval dominance, and
E-admissibility can be evaluated using the same stored expectations, but require
an additional selection or reporting convention before they define a scalar VOI.

The envelope, by contrast, is criterion-free. It reports the image of the credal
set under the classical VOI functional and is the right object when the question
is robustness of a precise-measure conclusion rather than choice under
imprecision. The framework should not be read as ordinary scenario analysis. The
target is the lower and upper values over the evidence-defined set. A finite grid
or finite list may approximate those endpoints, or may define the generators of a
convex credal set, but the endpoints carry a decision-theoretic interpretation
rather than being illustrative cases.

\subsection{Computation and limitations}
\label{sec:practical}
The computational cost is the product of estimating VOI under one represented
measure or generator and searching over \(\Pset\). EVPI is cheapest because its
endpoint structure is explicit. EVPPI is more demanding because conditional
values must be estimated for each represented measure or generator. EVSI is the
most expensive because of its preposterior calculation. A naive outer search can
therefore be slow and should be staged, with a coarse exploration to locate
candidate endpoint regions followed by refinement in those regions. Common random
numbers, reweighting, cached evaluations, surrogate models for the outer surface
\(h\mapsto T(P_h)\), and parallel computation reduce cost, and p-box
discretizations can be refined only where endpoints are sensitive.

Several limitations remain. The analyst must specify \(\Pset\). This is not
unique to the approach, since precise VOI also rests on assumptions, but the
credal-set analysis makes those assumptions explicit. The result is only as good
as a set broad enough to contain the defensible measures yet narrow enough to
exclude implausible ones. Computation is costly, especially for EVSI. Intervals
are less convenient than single numbers, although an interval is the honest
summary when the evidence does not identify one measure. The EVSI calculations
in the worked example are finite Monte Carlo regression approximations to
single-measure targets; the exact structural statements concern the EVPI
functionals and the credal-set definitions. Finally, the framework does not
replace model averaging or hierarchical modelling
\citep{price2011model,strong2012discrepancy,strong2014modelimprovement,koffijberg2018choices}.
When credible weights over structures can be defended, a model-averaged measure
is appropriate and is analysed by classical VOI. The credal-set treatment is for
the complementary case in which adding weights would answer a different question
and conceal the imprecision.

\subsection{Conclusion}
Classical VOI reports the value of information after a single measure has been
selected. Treating VOI as a functional on a credal set reports how that value
behaves over all measures compatible with the evidence and, under a stated
criterion for imprecision, how much information is worth to a decision maker. The
distinction matters whenever distributional form, dependence, model structure, or
expert judgement is not precisely identified. Reporting both quantities separates
conclusions that are robust across the credal set from conclusions that depend on
an unidentified measure.

\section*{Declaration of interests}
The author is employed by Medtronic Trading S\`arl. The views expressed are the
author's own. The author declares no other competing interests.

\section*{Funding}
No specific funding was received for this work.

\section*{Data availability}
The analysis code and derived data are included with this submission, in both
Python (the reference implementation) and R, and the reported values are
reproducible from a deterministic script. A tagged public release will be archived
on acceptance. The supplementary material gives the full estimation algorithm,
proofs of the estimator relations, and the explicit instances underlying the
structural results.

\section*{Acknowledgements}
The author thanks colleagues who provided feedback on earlier versions of this
work.

\appendix
\section{Proofs and supporting results}
\label{app:proofs}

\begin{proposition}[Expectation bounds for the full p-box band]
\label{prop:pbox-expectation}
Let $X$ be a bounded scalar output with support contained in $[\ell_X,u_X]$, and let $\underline F_X\le\overline F_X$ be valid cumulative distribution functions on $[\ell_X,u_X]$. Let $\mathcal B(\underline F_X,\overline F_X)$ be the full p-box band of all cumulative distribution functions $F$ supported on $[\ell_X,u_X]$ satisfying $\underline F_X(x)\le F(x)\le \overline F_X(x)$ for all $x$. Then the lower and upper expectations over this full p-box band are given by Equation~\eqref{eq:pbox-expectation}. If $\{F_{X,P}:P\in\Pset\}\subseteq\mathcal B(\underline F_X,\overline F_X)$, the inequalities in Equation~\eqref{eq:pbox-expectation-outer} follow.
\end{proposition}

\begin{proof}
For any cumulative distribution function $F$ supported on $[\ell_X,u_X]$,
\[
 \E_F(X)=\ell_X+\int_{\ell_X}^{u_X} \{1-F(x)\}\,dx.
\]
The expectation is decreasing in $F$ pointwise. Therefore the smallest expectation over the full band is obtained by using the largest admissible cumulative distribution function in the band, $\overline F_X$, and the largest expectation is obtained by using the smallest one, $\underline F_X$. This gives Equation~\eqref{eq:pbox-expectation}. If the output distributions induced by $\Pset$ form a proper subset of the band, optimizing over the larger band can only decrease the lower endpoint and increase the upper endpoint, which gives Equation~\eqref{eq:pbox-expectation-outer}. If instead $\underline F_X$ or $\overline F_X$ is only an envelope and is not itself an admissible cumulative distribution function, the integrals in Equation~\eqref{eq:pbox-expectation} are read as outer bounds rather than attained endpoints.
\end{proof}

\begin{proposition}[Lower-expectation EVPI]
\label{prop:robust-evpi}
For a finite strategy set $\A$, the lower-expectation EVPI in Equation~\eqref{eq:lower-evpi} is nonnegative.
\end{proposition}

\begin{proof}
Recall that $M(\theta)=\max_{a\in\A}\NB_a(\theta)$. For every strategy $a$ and every $\theta$, $M(\theta)\ge \NB_a(\theta)$. Monotonicity of expectation implies $\E_P M(\theta)\ge \E_P\{\NB_a(\theta)\}$ for every $P\in\Pset$. Taking infima over $P$ gives
\[
 \underline{\E}_{\Pset}M(\theta)
 \ge
 \underline{\E}_{\Pset}\{\NB_a(\theta)\}
\]
for every $a$. Taking the maximum over $a$ on the right gives
\[
 \underline{\E}_{\Pset}M(\theta)
 \ge
 \max_{a\in\A}\underline{\E}_{\Pset}\{\NB_a(\theta)\}.
\]
The difference between the two sides is Equation~\eqref{eq:lower-evpi}, so lower-expectation EVPI is nonnegative.
\end{proof}

\begin{proposition}[Sharp fixed-measure VOI envelope]
\label{prop:sharp}
Let $T(P)$ be a conventional fixed-measure VOI functional, such as EVPI, EVPPI for a prespecified parameter group, or EVSI for a prespecified design. The set
\[
 T_{\mathrm{env}}(\Pset)=\{T(P):P\in\Pset\}
\]
is the sharp set of fixed-measure VOI values left possible by the evidence. If $\Pset$ is compact and $T$ is continuous on $\Pset$, then the lower and upper endpoints are attained. If, in addition, $\Pset$ is connected, then $T_{\mathrm{env}}(\Pset)$ is an interval in $\mathbb R$.
\end{proposition}

\begin{proof}
For every $P\in\Pset$, the measure $P$ is admissible, and therefore $T(P)$ is a conventional VOI value that is compatible with the evidence. Conversely, if a number $v$ is not equal to $T(P)$ for any $P\in\Pset$, then no admissible measure can produce $v$ through the conventional VOI calculation. Thus $T_{\mathrm{env}}(\Pset)$ is exactly the set of fixed-measure VOI values left possible by the credal set.

If $\Pset$ is compact and $T$ is continuous, the extreme value theorem implies that $T$ attains its minimum and maximum on $\Pset$. If $\Pset$ is connected, the continuous image of a connected set in $\mathbb R$ is connected. The connected subsets of $\mathbb R$ are intervals. Therefore $T_{\mathrm{env}}(\Pset)$ is an interval. When $\Pset$ is not connected, the reported envelope is the smallest closed interval containing $T_{\mathrm{env}}(\Pset)$.
\end{proof}

The fixed-measure EVPI envelope is also continuous in total variation. By
Proposition~\ref{prop:sharp-tv}, under bounded net benefits
($|\NB_a(\theta)|\le B$ for all $a,\theta$) one has
$|\EVPI(P)-\EVPI(Q)|\le L\,\|P-Q\|_{\mathrm{TV}}$ with
$L\le\operatorname{osc}(M)+\max_a\operatorname{osc}(\NB_a)\le 4B$, so small changes
in the probability measure cannot produce arbitrarily large changes in
conventional EVPI.

\end{document}


\maketitle

This supplement provides four things. First, the full unified IP-VOI estimation
algorithm in pseudocode. Second, proofs of the estimator relations stated in the
main text, including the ordering $\EVSI\le\EVPPI\le\EVPI$ and the convergence of
the sample-information estimator to the partial-perfect-information ceiling. Third,
the variance-reduction argument for the common-random-number (CRN) estimator under
a normal--normal update. Fourth, the explicit four-point instance in which vertex
enumeration understates the partial-information envelope. Equation and
section numbers prefixed ``main'' refer to the main manuscript. Throughout,
exactness claims refer to population estimands unless a finite Monte Carlo or
regression approximation is explicitly being discussed. An executable reference
implementation accompanying this supplement reproduces every numerical value
reported in the worked application.

\tableofcontents

\section{Notation recalled}

Let $\A$ be a finite strategy set, $\theta$ the uncertain inputs, and
$\NB_a(\theta)$ the net benefit of strategy $a$, with $|\NB_a(\theta)|\le B$.
Write $M(\theta)=\max_{a\in\A}\NB_a(\theta)$. The credal set $\Pset$ is
represented computationally by a finite or discretized index set
$\Hset=\{h_1,\ldots,h_R\}$, where $h_m$ indexes an admissible measure or generator $P_{h_m}$.
For each $h_m$ we draw $\theta_m^{(i)}\sim P_{h_m}$, $i=1,\ldots,N$, and store
\[
 S_{m,i,a}=\NB_a\!\left(\theta_m^{(i)}\right),\qquad
 \widehat\mu_a(h_m)=\frac1N\sum_{i=1}^N S_{m,i,a}.
\]
Under the lower-expectation ($\Gamma$-maximin) rule, the current value, the
perfect-information value, and the partial/sample-information value are estimated
by
\begin{align}
 \widehat V_0^{\mathrm{LE}} &= \max_{a\in\A}\ \min_{1\le m\le R}\ \widehat\mu_a(h_m), \label{eq:s-hat-current}\\
 \widehat V_\Theta^{\mathrm{LE}} &= \min_{1\le m\le R}\ \frac1N\sum_{i=1}^N \max_{a\in\A} S_{m,i,a}, \label{eq:s-hat-perfect}\\
 \widehat V_Z^{\mathrm{LE}} &= \max_{\delta\in\Delta_Z}\ \min_{1\le m\le R}\ \frac1N\sum_{i=1}^N S_{m,i,\delta(z_m^{(i)})}, \label{eq:s-hat-policy}
\end{align}
with $\widehat{\EVPI}^{\mathrm{LE}}=\widehat V_\Theta^{\mathrm{LE}}-\widehat V_0^{\mathrm{LE}}$
and $\widehat{\VOI}_Z^{\mathrm{LE}}=\widehat V_Z^{\mathrm{LE}}-\widehat V_0^{\mathrm{LE}}$.
These are the empirical counterparts of Equations~(11), (12), and~(9) of the main
text.

\section{The unified algorithm}

\begin{algorithm}[h]
\caption{Unified IP-VOI computation (lower-expectation rule, with other rules
substituting the inner aggregation).}
\label{alg:s-unified}
\begin{algorithmic}[1]
\State \textbf{Admissible set.} Build $\Hset=\{h_1,\ldots,h_R\}$ representing
$\Pset$, recording the reference index $m_0$ if used. For finite discrete
imprecision, $\Hset$ enumerates the Cartesian product of the modelling choices:
exact if those measures are themselves the admissible family, whereas if their
convex closure is the intended credal set then $\Hset$ lists the generators and
nonlinear envelope endpoints may require additional optimization.
For a continuous component, $\Hset$ discretizes it on a grid (inner approximation).
\State \textbf{Inner simulation.} For each $h_m$ and $i=1,\ldots,N$, draw
$\theta_m^{(i)}\sim P_{h_m}$ and store $S_{m,i,a}$ for all $a\in\A$. For EVPPI/EVSI
also store the observable $z_m^{(i)}=\phi(\theta_m^{(i)})$ or a simulated dataset
$z_m^{(i)}=Y_d^{(i)}$.
\State \textbf{Current value.} Compute $\widehat\mu_a(h_m)$ and
$\widehat V_0^{\mathrm{LE}}$ via \eqref{eq:s-hat-current}.
\State \textbf{Estimand.}
\Statex \quad EVPI: compute $\widehat V_\Theta^{\mathrm{LE}}$ via
\eqref{eq:s-hat-perfect} and set $\widehat{\EVPI}^{\mathrm{LE}}
=\widehat V_\Theta^{\mathrm{LE}}-\widehat V_0^{\mathrm{LE}}$.
\Statex \quad EVPPI/EVSI: solve the finite policy search \eqref{eq:s-hat-policy}
(over the full policy class, or a provably safe reduced class for the model at hand) and set
$\widehat{\VOI}_Z^{\mathrm{LE}}=\widehat V_Z^{\mathrm{LE}}-\widehat V_0^{\mathrm{LE}}$.
For EVSI subtract the study cost.
\State \textbf{Envelope.} For the fixed-measure envelope, evaluate any validated
single-measure estimator $\widehat T(\cdot)$ at each $h_m$ and report
$[\min_m \widehat T(h_m),\ \max_m \widehat T(h_m)]$ with the reference value
$\widehat T(h_{m_0})$. This is exact when $\Hset$ enumerates the admissible
measures. When $\Hset$ lists the generators of a convex credal set, it is exact
for the EVPI \emph{lower} endpoint (Theorem~1(b) of the main text). The EVPI
\emph{upper} endpoint is obtained from the linear program of Remark~2, and
EVPPI/EVSI generator ranges are inner approximations unless the condition of
Proposition~2(b) holds (and then only for the upper endpoint).
\State \textbf{Stability report.} Report the reference value, the endpoint values
and the assumptions producing them, the threshold classification, and the Monte
Carlo error at candidate endpoints.
\end{algorithmic}
\end{algorithm}

\section{Proofs of the estimator relations}

Throughout this section fix a single measure $P$. The relations below hold under
each $h_m$ and therefore propagate to the fixed-measure envelope endpoints. They
do not, however, transfer to the rule-specific lower-expectation quantities by an
outer $\min$/$\max$: the $\Gamma$-maximin VOI need not equal the infimum over $P$
of a single-measure VOI and can exceed the entire envelope, as shown in the main
text. An ordering for the rule-specific values, where needed, follows instead from
nested policy and information classes under the relevant garbling assumptions. Write
$\mu_a=\E_P[\NB_a]$, and recall the population quantities
$\EVPI(P)=\E_P[M]-\max_a\mu_a$,
$\EVPPI(P;\phi)=\E_P[\max_a\E_P(\NB_a\mid\phi)]-\max_a\mu_a$, and
$\EVSI(P;d)=\E_P[\max_a\E_P(\NB_a\mid Y_d)]-\max_a\mu_a$.

\subsection{\texorpdfstring{The ordering $0\le\EVSI\le\EVPPI\le\EVPI$}{The ordering EVSI <= EVPPI <= EVPI}}

\begin{proposition}\label{prop:s-order}
Let the partial-perfect-information signal be a \emph{deterministic} function
$\phi=\phi(\theta)$ of the uncertain parameters, and let the future dataset $Y_d$
be conditionally independent of $\theta$ given $\phi$ (equivalently, $Y_d$ depends
on $\theta$ only through $\phi$, a garbling of $\phi$). Then
\[
 0\ \le\ \EVSI(P;d)\ \le\ \EVPPI(P;\phi)\ \le\ \EVPI(P).
\]
\end{proposition}

\begin{proof}
All three VOI quantities share the same current value $\max_a\mu_a$, so it
suffices to order the three post-information values
\[
\begin{aligned}
&V_0=\max_a\mu_a,\qquad
 V_{\text{EVSI}}=\E[\max_a\E(\NB_a\mid Y_d)],\\
&V_{\text{EVPPI}}=\E[\max_a\E(\NB_a\mid\phi)],\qquad
 V_{\text{EVPI}}=\E[\max_a\NB_a].
\end{aligned}
\]
We use repeatedly the following monotonicity fact. For any two
$\sigma$-algebras $\mathcal G_1\subseteq\mathcal G_2$ and any integrable random
vector $(W_a)_a$, Jensen's inequality applied to the convex map
$x\mapsto\max_a x_a$ together with the tower property gives
\begin{equation}\label{eq:s-mono}
 \E\!\big[\max_a\E(W_a\mid\mathcal G_1)\big]
 \ \le\ \E\!\big[\max_a\E(W_a\mid\mathcal G_2)\big],
\end{equation}
since $\max_a\E(W_a\mid\mathcal G_1)=\max_a\E[\E(W_a\mid\mathcal G_2)\mid\mathcal G_1]
\le\E[\max_a\E(W_a\mid\mathcal G_2)\mid\mathcal G_1]$ pointwise, and taking
expectations removes the outer conditioning.

\emph{Step 1 ($V_0\le V_{\text{EVSI}}$).} Apply \eqref{eq:s-mono} with
$W_a=\NB_a$, $\mathcal G_1=\{\emptyset,\Omega\}$ the trivial $\sigma$-algebra on
the underlying space $\Omega$ carrying both $\theta$ and $Y_d$, and
$\mathcal G_2=\sigma(Y_d)$.
The left-hand side is $\max_a\E(\NB_a)=V_0$ and the right-hand side is
$V_{\text{EVSI}}$. Nonnegativity of $\EVSI$ is exactly this inequality.

\emph{Step 2 ($V_{\text{EVSI}}\le V_{\text{EVPPI}}$).} This step is \emph{not} a
plain $\sigma$-algebra inclusion: a garbling $Y_d$ of $\phi$ is generated from
$\phi$ through an independent noise mechanism and is in general \emph{not}
$\sigma(\phi)$-measurable, so $\sigma(Y_d)\subseteq\sigma(\phi)$ fails. We use the
conditional-independence hypothesis instead. Because $\NB_a=\NB_a(\theta)$ and
$Y_d\perp\theta\mid\phi$, we have $\NB_a\perp Y_d\mid\phi$, and therefore
\[
 \E(\NB_a\mid Y_d)
 =\E\!\big[\E(\NB_a\mid\phi,Y_d)\,\big|\,Y_d\big]
 =\E\!\big[\E(\NB_a\mid\phi)\,\big|\,Y_d\big]
 =\E(\eta_a\mid Y_d),\qquad \eta_a:=\E(\NB_a\mid\phi).
\]
Hence, by Jensen for $x\mapsto\max_a x_a$ and the tower property,
\[
 V_{\text{EVSI}}=\E\big[\max_a\E(\eta_a\mid Y_d)\big]
 \le\E\big[\E(\max_a \eta_a\mid Y_d)\big]
 =\E[\max_a \eta_a]
 =\E[\max_a\E(\NB_a\mid\phi)]=V_{\text{EVPPI}}.
\]

\emph{Step 3 ($V_{\text{EVPPI}}\le V_{\text{EVPI}}$).} Because $\phi=\phi(\theta)$
is deterministic, $\sigma(\phi)\subseteq\sigma(\theta)$, and
$\E(\NB_a\mid\theta)=\NB_a$. Apply \eqref{eq:s-mono} with $W_a=\NB_a$,
$\mathcal G_1=\sigma(\phi)$ and $\mathcal G_2=\sigma(\theta)$: the left-hand side
is $V_{\text{EVPPI}}$ and the right-hand side is $\E[\max_a\NB_a]=V_{\text{EVPI}}$.

Combining the three steps gives
$V_0\le V_{\text{EVSI}}\le V_{\text{EVPPI}}\le V_{\text{EVPI}}$. Subtracting the
common $V_0$ proves the stated chain.
\end{proof}

\begin{lemma}[Consistency of the single-measure estimators, with the ordering as a numerical check]
\label{lem:s-emp-order}
Fix an admissible measure $P=h_m$. The EVPI plug-in $\widehat\EVPI(h_m)$ is
consistent for $\EVPI(P)$ as $N\to\infty$ (a difference of sample averages of
bounded functionals). The EVPPI and EVSI plug-ins
$\widehat\EVPPI(h_m),\widehat\EVSI(h_m)$ are consistent for the
conditional-expectation \emph{approximations} actually computed---a fixed-bin
estimate of $\E_P(\NB_a\mid\phi)$ for EVPPI and a regression/posterior-mean
estimate for EVSI---and hence for the corresponding approximate VOI targets.
Consistency for the exact population EVPPI/EVSI additionally requires a consistent
nonparametric conditional-mean estimator, a discretization whose mesh is refined
with $N$, or a correctly specified regression class. The exact population targets satisfy the
ordering $0\le\EVSI\le\EVPPI\le\EVPI$ of Proposition~\ref{prop:s-order}. For exact
population targets, a pointwise ordering $T_1(P)\le T_2(P)$ holding for every
admissible measure $P$ is inherited by the corresponding fixed-measure envelope
endpoints computed over the same set: when the $h_m$ enumerate the admissible
measures this is the finite $\min_m T(h_m)$/$\max_m T(h_m)$, whereas when the $h_m$
are generators of a convex credal set the ordering still holds over the hull, but
the endpoints must be computed by the appropriate structural or numerical
optimization (for the EVPI upper endpoint, the linear program of the main text,
whose plug-in is consistent by Proposition~\ref{prop:s-lp-consistency} below), not
by generator extrema alone. The
fixed-bin EVPPI and
regression/posterior-mean EVSI \emph{approximation} targets, however, need not
satisfy the exact ordering automatically unless the conditional-mean approximations
preserve the nested information structure. We therefore treat the ordering as a
numerical validation check on the estimates rather than a guaranteed property of
the approximations. This is a statement about the fixed-measure quantities only:
the rule-specific lower-expectation values are \emph{not} ordered by this argument,
since they are not an outer $\min$/$\max$ of a single-measure VOI (see the remarks
opening this section). The finite-sample plug-in $\E[\max_a\cdot]$ terms are
biased upward (Jensen), so a naive estimator can violate the ordering for small
$N$. The CRN construction of Section~\ref{sec:s-crn} reduces the variance of the
post-information/current-value difference by differencing against shared draws, but
it does not remove the maximization (Jensen) bias, which vanishes only as
$N\to\infty$. The validation study confirms that the chosen approximations recover
the known EVPI and EVPPI-ceiling limits in the application and that the estimates
satisfy the ordering within Monte Carlo error across all 12 generators.
\end{lemma}

\begin{proof}
The EVPI term is a difference of sample averages of bounded functionals of
i.i.d.\ draws, hence consistent by the strong law of large numbers, with the outer
$\min$/$\max$ over the finite index set $\Hset$ continuous. The EVPPI and EVSI
plug-ins replace the exact conditional expectation $\E_P(\NB_a\mid\cdot)$ by a
fixed-bin or regression estimate. For a fixed bin count or regression class these
converge as $N\to\infty$ to the corresponding approximate conditional means, and
hence to the approximate VOI targets, and to the exact targets only if the bin
mesh shrinks or the regression is correctly specified. The upward bias of
$\frac1N\sum_i\max_a(\cdot)$ relative to $\max_a\frac1N\sum_i(\cdot)$ is the
standard maximization bias and is $O(N^{-1/2})$ in the present bounded setting.
\end{proof}

\begin{proposition}[Consistency of the EVPI upper-endpoint linear program]\label{prop:s-lp-consistency}
Let $D_{k,a}=\E_{G_k}[M]-\E_{G_k}[\NB_a]$ for $k=1,\ldots,K$ and $a\in\A$, and let
\[
 U(D)=\max_{w\in\Delta_K}\ \min_{a\in\A}\ \sum_{k=1}^K w_k D_{k,a}
\]
be the EVPI upper endpoint over the finitely generated credal set
$\operatorname{conv}\{G_1,\ldots,G_K\}$ (the linear program of the main text). For
any array $\widehat D$ with
$\|\widehat D-D\|_\infty=\max_{k,a}|\widehat D_{k,a}-D_{k,a}|\le\varepsilon$,
\[
 \bigl|U(\widehat D)-U(D)\bigr|\le\varepsilon.
\]
Consequently, since each $\E_{G_k}[M]$ and $\E_{G_k}[\NB_a]$ is estimated
consistently by the Monte Carlo sample average of a bounded functional, the plug-in
endpoint $U(\widehat D)$ is consistent for $U(D)$, with Monte Carlo error
controlled by the worst per-generator, per-action error.
\end{proposition}

\begin{proof}
Fix $w\in\Delta_K$ and $a\in\A$. Since $w_k\ge0$ and $\sum_k w_k=1$,
\[
 \Bigl|\textstyle\sum_k w_k\widehat D_{k,a}-\sum_k w_k D_{k,a}\Bigr|
 \le\sum_k w_k\,|\widehat D_{k,a}-D_{k,a}|\le\varepsilon .
\]
Thus the affine maps $w\mapsto\sum_k w_k\widehat D_{k,a}$ and
$w\mapsto\sum_k w_k D_{k,a}$ differ by at most $\varepsilon$ pointwise in $a$. The
operations $\min_a$ and $\max_w$ are each $1$-Lipschitz for the supremum norm (a
minimum, or maximum, of functions that differ pointwise by at most $\varepsilon$
differ by at most $\varepsilon$), so applying $\min_a$ and then $\max_w$ preserves
the bound, giving $|U(\widehat D)-U(D)|\le\varepsilon$. Letting $\varepsilon\to0$
as $N\to\infty$ yields consistency.
\end{proof}

\subsection{Convergence of EVSI to the EVPPI ceiling}

\begin{proposition}\label{prop:s-evsi-limit}
Consider a future study that estimates the partial parameter (here the
log-odds-ratio $\ell$) through a statistic whose sampling variance
$\sigma^2(n)\to0$ as the planned sample size $n\to\infty$, and suppose the
conditional expected net benefits $\ell\mapsto\E_P(\NB_a\mid\ell)$ are
$P$-almost surely continuous (or more generally, the posterior expectations converge in
$L^1$ as the posterior collapses). Then
\[
 \lim_{n\to\infty}\EVSI(P;n)=\EVPPI(P;\ell).
\]
\end{proposition}

\begin{proof}
As $n\to\infty$ the sampling variance $\sigma^2(n)\to0$, so the future statistic
becomes perfectly informative about $\ell$ and the posterior law of $\ell$ given
the data concentrates at the realized value. Hence the data-conditional
expectation $\E_P(\NB_a\mid Y_d)$ converges to the $\ell$-conditional expectation
$\E_P(\NB_a\mid\ell)$ for each $a$, $P$-almost surely. By the assumed $P$-almost
sure continuity of $\ell\mapsto\E_P(\NB_a\mid\ell)$, the post-information
\emph{value} $\max_a\E_P(\NB_a\mid Y_d)$ converges to $\max_a\E_P(\NB_a\mid\ell)$,
the value attained by an action optimal given $\ell$. The identity of the optimal
action need not converge at ties, but its value does. Bounded net benefits
($|\NB_a|\le B$) give dominated convergence of the outer expectation, so
$\EVSI(P;n)\to\EVPPI(P;\ell)$. The normal--normal update used in the application is
the concrete instance: the posterior mean of $\ell$ given the trial statistic
$\hat y$ is $\E(\ell\mid\hat y)=\kappa(n)\,\hat y+(1-\kappa(n))\,\mu_0$ with
$\kappa(n)=\sigma_0^2/(\sigma_0^2+\sigma^2(n))\to1$ and prior mean and variance
$\mu_0,\sigma_0^2$, while the posterior variance
$\sigma_0^2\sigma^2(n)/(\sigma_0^2+\sigma^2(n))\to0$, so the posterior collapses to a point
mass at $\ell$. The reverse trivial bound
$\EVSI(P;n)\le\EVPPI(P;\ell)$ for all $n$ is Proposition~\ref{prop:s-order}. Hence
the limit is approached from below, consistent with Figure~2(b) of the main text.
\end{proof}

\section{The common-random-number estimator}
\label{sec:s-crn}

The EVSI estimator differences a post-data decision value against a current value.
A naive nested estimator draws fresh inner samples for each simulated dataset,
which is both upward-biased (Lemma~\ref{lem:s-emp-order}) and high-variance. We
instead reuse a single pool of draws. Reusing one prior-predictive pool to couple
the with- and without-information terms is the standard common-random-number
(control-variate) device for Monte Carlo variance reduction, and the single-pool
reuse it relies on is already built into regression-based VOI estimators such as
those of Strong et~al.\ and Heath et~al. We record only the property that makes
the coupling exact for this estimator.

\begin{lemma}[CRN coupling under data-independent posterior variance]\label{lem:s-crn}
Under the normal--normal update, the posterior variance of the partial parameter
does not depend on the realized data. Consequently, holding fixed a single pool of
draws of all other parameters and of the standardized partial-parameter
perturbation, and varying only the posterior mean across simulated datasets,
reproduces the analytic normal--normal preposterior law of the posterior mean on
which the plug-in decision rule acts, up to Monte Carlo error from the finite
simulation pool. The coupling is exact for that analytic posterior-mean
law. The resulting EVSI estimate nonetheless remains a regression-based
posterior-mean plug-in (the post-data action is taken at the fitted conditional
means evaluated at the posterior mean rather than under the full posterior), so it
is a finite Monte Carlo approximation rather than exact normal--normal EVSI.
\end{lemma}

\begin{proof}
Write the partial parameter as $\ell$ with prior $\ell\sim\mathcal N(\mu_0,\sigma_0^2)$
and trial likelihood $\hat y\mid\ell\sim\mathcal N(\ell,\sigma^2)$. The posterior is
$\ell\mid\hat y\sim\mathcal N\!\big(\kappa\hat y+(1-\kappa)\mu_0,\ \sigma_0^2\sigma^2/(\sigma_0^2+\sigma^2)\big)$
with $\kappa=\sigma_0^2/(\sigma_0^2+\sigma^2)$. The posterior variance
$v=\sigma_0^2\sigma^2/(\sigma_0^2+\sigma^2)$ is a function of $(\sigma_0^2,\sigma^2)$ only. It
does not depend on $\hat y$. The preposterior law of the posterior mean is induced
by the marginal $\hat y\sim\mathcal N(\mu_0,\sigma_0^2+\sigma^2)$, equivalently by
$\hat y=\ell+\sqrt{\sigma^2}\,\xi$ with $\xi\sim\mathcal N(0,1)$
independent of the pool. Holding the pool (the draws of $\ell$ and of all other
parameters) fixed and drawing only $\xi$ therefore generates the
preposterior distribution of the posterior-mean decision rule under the empirical
prior pool, since no other quantity entering the decision depends on $\hat y$
except through the posterior mean. As the pool size grows this converges to the
analytic normal--normal preposterior law. Differencing the post-data value and the current value against the same pool
cancels much of the shared simulation noise, leaving the signal-driven variation as
the main remaining source. This is a control-variate construction, exact in
expectation under the empirical-pool target, and it reduced the between-replicate
variance in the worked application.
\end{proof}

\noindent\emph{Empirical magnitude.} At the reference measure with $N=20000$ draws
and the $n=150$ per-arm design, the standard deviation of the EVSI estimate across
$40$ independent replications was about $\pounds 3.7$ under the common-random-number
coupling and about $\pounds 11$ when the current value was instead differenced
against an independent pool (both centred near $\pounds 240$), a roughly threefold
reduction. This check is reproducible via \texttt{crn\_variance\_check.py} in the
reproducibility folder.

\section{Explicit four-point instance: vertex enumeration understates the EVPPI
envelope}
\label{sec:s-vertex}

This section gives the explicit instance referenced in
Proposition~2(a) of the main text: a finitely parameterized family on four
states for which the EVPPI, as a function of a single mixing parameter, is zero at
both endpoints of its range but strictly positive in the interior, so vertex
(endpoint) enumeration returns zero and misses the maximum.

\paragraph{Construction.} Two actions, $a\in\{0,1\}$, an observed signal
$z\in\{0,1\}$ and an unobserved nuisance state $s\in\{0,1\}$, giving four states
$(z,s)\in\{(0,0),(0,1),(1,0),(1,1)\}$ in this order. Net benefits are
\[
 \NB_0=(0,0,0,0),\qquad \NB_1=(1,-2,3,-1).
\]
The signal marginal is fixed at $P(z=0)=\tfrac12$, and a single parameter
$t\in[0,1]$ controls the within-signal weighting:
$P(s=1\mid z=0)=t$ and $P(s=1\mid z=1)=t$, so the joint mass vector is
\[
 P(t)=\big(\tfrac12(1-t),\ \tfrac12 t,\ \tfrac12(1-t),\ \tfrac12 t\big).
\]

\paragraph{EVPPI as a function of $t$.} The current value is
$V_0(t)=\max\{\E_{P(t)}[\NB_0],\E_{P(t)}[\NB_1]\}=\max\{0,\ \E_{P(t)}[\NB_1]\}$,
and the post-signal value is
\[
 W(t)=\tfrac12\max\big\{\E(\NB_0\mid z{=}0),\E(\NB_1\mid z{=}0)\big\}
      +\tfrac12\max\big\{\E(\NB_0\mid z{=}1),\E(\NB_1\mid z{=}1)\big\},
\]
with $\E(\NB_a\mid z{=}0)=(1-t)\NB_a^{(0,0)}+t\,\NB_a^{(0,1)}$ and similarly for $z{=}1$.
Then $\EVPPI(t)=W(t)-V_0(t)$.

\paragraph{Values.} Direct evaluation gives $\EVPPI(0)=\EVPPI(1)=0$. The function
is piecewise linear in $t$: it is $0$ on $[0,\tfrac13]$, rises with slope
$\tfrac32$ on $[\tfrac13,\tfrac47]$, falls with slope $-2$ on $[\tfrac47,\tfrac34]$,
and is $0$ on $[\tfrac34,1]$. Its maximum is therefore attained at the interior
kink $t=\tfrac47$, with the exact value
\[
 \max_{t\in[0,1]}\EVPPI(t)=\EVPPI\!\big(\tfrac47\big)=\tfrac{5}{14}\approx0.357.
\]
(A coarse grid reports $\approx0.355$ near $t\approx0.57$, slightly understating the
exact peak.) Vertex enumeration, which evaluates only the endpoints
$t\in\{0,1\}$, returns
$0$ and therefore \emph{misses} the maximum entirely. This is the failure of
vertex attainment proved in general in Proposition~2(a) of the main text. The
post-information value is convex on the credal set regardless of the signal
marginal (a supremum over policies of affine functionals). What restores convexity
of $\VOI_Z$, and hence vertex attainment of its upper endpoint, is a single
pre-information optimal action common to all admissible measures, which makes the
current value affine (Proposition~2(b)). The lower endpoint may still require an
interior search. In the instance above this condition fails, because which action
is pre-information optimal switches with $t$.

\paragraph{Reproduction.} The reference implementation evaluates $\EVPPI(t)$ on a
grid and confirms the interior maximum. The same routine produces the
$\Gamma$-maximin overshoot instance (two states, $\NB_0=(1,-1)$, $\NB_1=(-1,1)$,
envelope $[0,1]$, $\EVPI_{\Pset}^{\mathrm{LE}}=2$) used in
Theorem~1(c) of the main text. Both are exact arithmetic and require no Monte Carlo.